\documentclass[preprint,12pt]{elsarticle}

\usepackage{hyperref}
\usepackage{booktabs}
\usepackage{threeparttable}
\usepackage{tabularx}
\usepackage{float}
\usepackage{adjustbox}
\usepackage{subcaption}
\usepackage{sidecap}
\usepackage{amsthm}
\usepackage{rotating}
\usepackage{amssymb}
\usepackage{amsmath}
\usepackage{multirow}
\usepackage{booktabs}
\usepackage{placeins}
\usepackage{amsmath}
\usepackage{algorithm}
\usepackage{algorithmic}
\usepackage[margins]{trackchanges}

\addeditor{YJ}
\addeditor{DAW}
\addeditor{MBD}
\usepackage{xspace}
\makeatletter
\@namedef{ver@everyshi}{}
\makeatother

\usepackage{tikz}
\usetikzlibrary{bayesnet}
\usetikzlibrary{arrows}
\usetikzlibrary{shapes.multipart}
\tikzset{
state/.style={
       rectangle split,
       rectangle split parts=2,
       rectangle split part fill={red!30,blue!20},
       rounded corners,
       draw=black, very thick,
       minimum height=2em,
       text width=3cm,
       inner sep=2pt,
       text centered,
       }
}

\usetikzlibrary{shapes}
\usetikzlibrary{fit}
\usetikzlibrary{chains}
\usetikzlibrary{arrows}

\newtheorem{prop}{Proposition}

\newcommand{\y}{\mathbf{y}}
\newcommand{\x}{\mathbf{x}}
\newcommand{\z}{\mathbf{z}}

\DeclareMathOperator{\Var}{\widehat{Var}}

\newcommand{\errpm}[2]{#1 {\scriptsize $\pm$ #2}}

\newcommand{\ELBO}{\mathcal{L}}
\newcommand{\ELBOK}{\mathcal{L}^{>k}}

\newcommand{\LLRK}{\mathbf{\mathcal{L}\mathcal{L}\mathcal{R}}^{>k}}

\DeclareMathOperator*{\argmax}{arg\,max}

\tikzstyle{latent} = [circle,fill=white,draw=black,inner sep=1pt,
minimum size=20pt, font=\fontsize{10}{10}\selectfont, node distance=1]
\tikzstyle{obs} = [latent,fill=gray!25]
\tikzstyle{const} = [rectangle, inner sep=0pt, node distance=1]
\tikzstyle{factor} = [rectangle, fill=black,minimum size=5pt, inner
sep=0pt, node distance=0.4]
\tikzstyle{det} = [latent, diamond]

\tikzstyle{plate} = [draw, rectangle, rounded corners, fit=#1]
\tikzstyle{wrap} = [inner sep=0pt, fit=#1]
\tikzstyle{gate} = [draw, rectangle, dashed, fit=#1]

\tikzstyle{caption} = [font=\footnotesize, node distance=0] %
\tikzstyle{plate caption} = [caption, node distance=0, inner sep=0pt,
below left=5pt and 0pt of #1.south east] %
\tikzstyle{factor caption} = [caption] %
\tikzstyle{every label} += [caption] %

\tikzset{>={triangle 45}}

\renewcommand{\edge}[3][]{ %
  \foreach \x in {#2} { %
    \foreach \y in {#3} { %
      \path (\x) edge [->,#1] (\y) ;%
    } ;
  } ;
}

\title{Optimizing Latent Dimension Allocation in Hierarchical
VAEs:
Balancing Attenuation and Information Retention for OOD
Detection}

\author[UVA]{Dane Williamson}
\author[UVA]{Yangfeng Ji}
\author[UVA]{Matthew Dwyer}
\address[UVA]{University of Virginia}

\begin{document}
\begin{frontmatter}

% \maketitle
\begin{abstract}
Out-of-distribution (OOD) detection is a critical task in machine learning, particularly for safety-critical applications where unexpected inputs must be reliably flagged. While hierarchical variational autoencoders (HVAEs) offer improved representational capacity over traditional VAEs, their performance is highly sensitive to how latent dimensions are distributed across layers. Existing approaches often allocate latent capacity arbitrarily, leading to ineffective representations or posterior collapse. In this work, we introduce a theoretically grounded framework for optimizing latent dimension allocation in HVAEs, drawing on principles from information theory to formalize the trade-off between information loss and representational attenuation. We prove the existence of an optimal allocation ratio $r^{\ast}$ under a fixed latent budget, and empirically show that tuning this ratio consistently improves OOD detection performance across datasets and architectures. Our approach outperforms baseline HVAE configurations and provides practical guidance for principled latent structure design, leading to more robust OOD detection with deep generative models.
\end{abstract}
    
\end{frontmatter}
\section{Introduction}
\label{sec:intro}
\par Out-of-distribution (OOD) detection plays a crucial role in many real-world scenarios by identifying data points that deviate from the norm, potentially indicating fraud, faults, or invalid data. 
Traditional methods \citep{liu2020energyOOD, chang2020godin, song2022rankFeat, hendrycks2016baseline} often rely on statistical thresholds or predefined rules, which may not capture complex anomalies effectively.
On the other hand, generative models, particularly Variational Autoencoders (VAEs) \citep{kingma2013vae}, have shown promising performance in capturing complex data distributions and generating realistic samples. 

\par One significant advancement along this line is the use of hierarchical VAEs (HVAEs) for out-of-distribution detection, leveraging their ability to detect anomalies in the latent space representations of data.
Hierarchical Variational Autoencoders (HVAEs) encode data into a structured hierarchy of latents that reflects the underlying generation process \citep{vahdat2021nvae, sonderby2016laddervae}. This multi-level structure allows HVAEs to capture data distributions at varying levels of abstraction, offering significant advantages for out-of-distribution (OOD) detection compared to traditional VAEs.

\par A critical challenge for Deep Generative Models (DGMs) is their potential over-reliance on features that may not be semantically meaningful, such as color or texture differences between images \citep{havtorn2021hvaes}. This can undermine OOD detection performance. HVAEs mitigate this issue by leveraging semantically relevant features across the hierarchical latent structure, making them a powerful tool for robust OOD detection. 

\par By modeling global patterns, like class distinctions, in higher-level latents and finer details, such as texture or lighting, in lower-level latents, HVAEs disentangle semantically meaningful variations. This tradeoff between abstraction and noise is clearly visualized in ~\autoref{fig:reconstructed_images}, which illustrates how latent configuration impacts semantic quality in reconstructed samples. This enables the model to perform density estimation based on these meaningful features, improving sensitivity to OOD deviations \citep{ren2019llrood, nalisnick2019dodeep}.

\par The hierarchical design of HVAEs provides flexibility in balancing model capacity and generalization, which is crucial when handling diverse or complex data distributions.

\par Motivated by the demonstrated potential of hierarchical VAEs for OOD detection \citep{havtorn2021hvaes, li2022adaptiveratio}, our work seeks to enhance their performance. Despite their promise, HVAEs face two persistent challenges. First, they often exhibit \textit{suboptimal latent space representations}, where deeper latent variables become decoupled from the input data, a form of posterior collapse \citep{bowman2015generatingsentences, kingma2016improvingvi, xi2016vlossae, dieng2019skipposterior, sonderby2016laddervae}. These are measured both theoretically, via diminished mutual information $I(X; Z_i)$ between latent variables and the input, and empirically, by degraded performance on OOD detection metrics. Second, HVAEs suffer from \textit{limited scalability} in architecture design: their performance depends heavily on how latent capacity is distributed across layers, but prior work offers little guidance on this front \citep{brock2019largescalegan, ho2020ddpm}. We address both issues by introducing a theoretically grounded allocation scheme that preserves semantically meaningful information across the hierarchy while avoiding overcompression.

\begin{SCfigure}[50][ht]
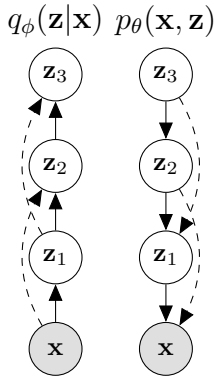

    \tikz{
        \node[obs] (x_inf) {$\x$};%
        \node[latent,above=.5cm of x_inf](z1_inf){$\z_1$}; %
        \node[latent,above=.5cm of z1_inf](z2_inf){$\z_2$}; %
        \node[latent,above=.5cm of z2_inf](z3_inf){$\z_3$}; %
        \node[above=of z3_inf, yshift=-1.cm] (phi) {$q_\phi(\z|\x)$}; 
        
        \edge[]{x_inf}{z1_inf};
        \edge[]{z1_inf}{z2_inf};
        \edge[]{z2_inf}{z3_inf};
        \edge[dashed, bend left]{x_inf}{z2_inf};
        \edge[dashed, bend left]{z1_inf}{z3_inf};
        
        \node[obs,right=0.75cm of x_inf] (x_gen) {$\x$};%
        \node[latent,above=.5cm of x_gen](z1_gen){$\z_1$}; %
        \node[latent,above=.5cm of z1_gen](z2_gen){$\z_2$}; %
        \node[latent,above=.5cm of z2_gen](z3_gen){$\z_3$}; %
        \node[above=of z3_gen, yshift=-1.cm] (theta) {$p_\theta(\x,\z)$}; 
        
        \edge[]{z2_gen}{z1_gen};
        \edge[]{z1_gen}{x_gen};
        \edge[]{z3_gen}{z2_gen};
        \edge[dashed, bend left]{z2_gen}{x_gen};
        \edge[dashed, bend left]{z3_gen}{z1_gen};
    }
    \vspace{-5mm}
    \caption{The encoder, $q_{\phi} (z \vert x)$,  and decoder, $p_{\theta}(x, z)$, models for a 3 level ($L$) HVAE. Solid lines represent stochastic dependencies in the data generation process, dashed lines indicate skip connections which assist in alleviating posterior collapse.}
    \label{fig:hvae_standard}
    \vspace{5mm}
\end{SCfigure}
\vspace{-10pt}

\par In particular, we introduce a structured allocation bias that encourages meaningful information flow through all latent levels, which can help reduce the severity of posterior collapse. We empirically demonstrate that tuning the relative dimensionality of latent layers, under a fixed capacity budget, is associated with improved latent utilization and enhanced OOD sensitivity, without increasing model size.

\par Despite extensive work on generative models for OOD detection, there has been limited focus on selecting appropriate latent architectures or investigating how structural design impacts model performance. Prior work \citep{havtorn2021hvaes, li2022adaptiveratio, maaloe2019biva} does not explicitly address latent allocation strategies, and existing regularization techniques \citep{dieng2021consistencyregularizationvae, bozkurt2021evaluatingcombvae, xu2020learningautoencoders, ma2018constrained, wu2020vector} are not tailored for improving OOD detection.

\par In this paper, we propose a theoretically motivated framework for optimizing latent dimension allocation in HVAEs, grounded in the Information Bottleneck principle. By formalizing the trade-off between representational attenuation and information loss, we design allocation strategies that enhance the robustness and effectiveness of OOD detection models. The main contributions of this paper are:
\begin{itemize}
    \item We introduce a closed-form latent allocation strategy for HVAEs that distributes a fixed dimensionality budget across the hierarchy using a geometric progression, and we prove the existence of an optimal reduction ratio $r^{\ast}$.
    \item We provide theoretical analysis linking latent dimensionality with generalization bounds and mutual information, showing that suboptimal allocations result in either overcompression or attenuation.
    \item We empirically validate our approach across diverse datasets, demonstrating that tuning $r$ significantly improves OOD detection metrics (AUROC, AUPRC, FPR80, FPR95) over baseline configurations.
    \item We show that the optimal configuration identified by $r^{\ast}$ generalizes across OOD pairings for a given in-distribution dataset, offering a principled guideline for HVAE architecture design.
\end{itemize}

\par We substantiate these contributions through theoretical analysis in ~\autoref{sec:approach_latent_dimension_ratio} and \autoref{sec:tradeoff_latent_dimension_ratio}, and empirical validation in ~\autoref{sec:results}.

\begin{figure}[t]
    \centering
    % Row for Original Images
    \begin{subfigure}[t]{0.15\textwidth}
        \centering
        \includegraphics[width=\textwidth]{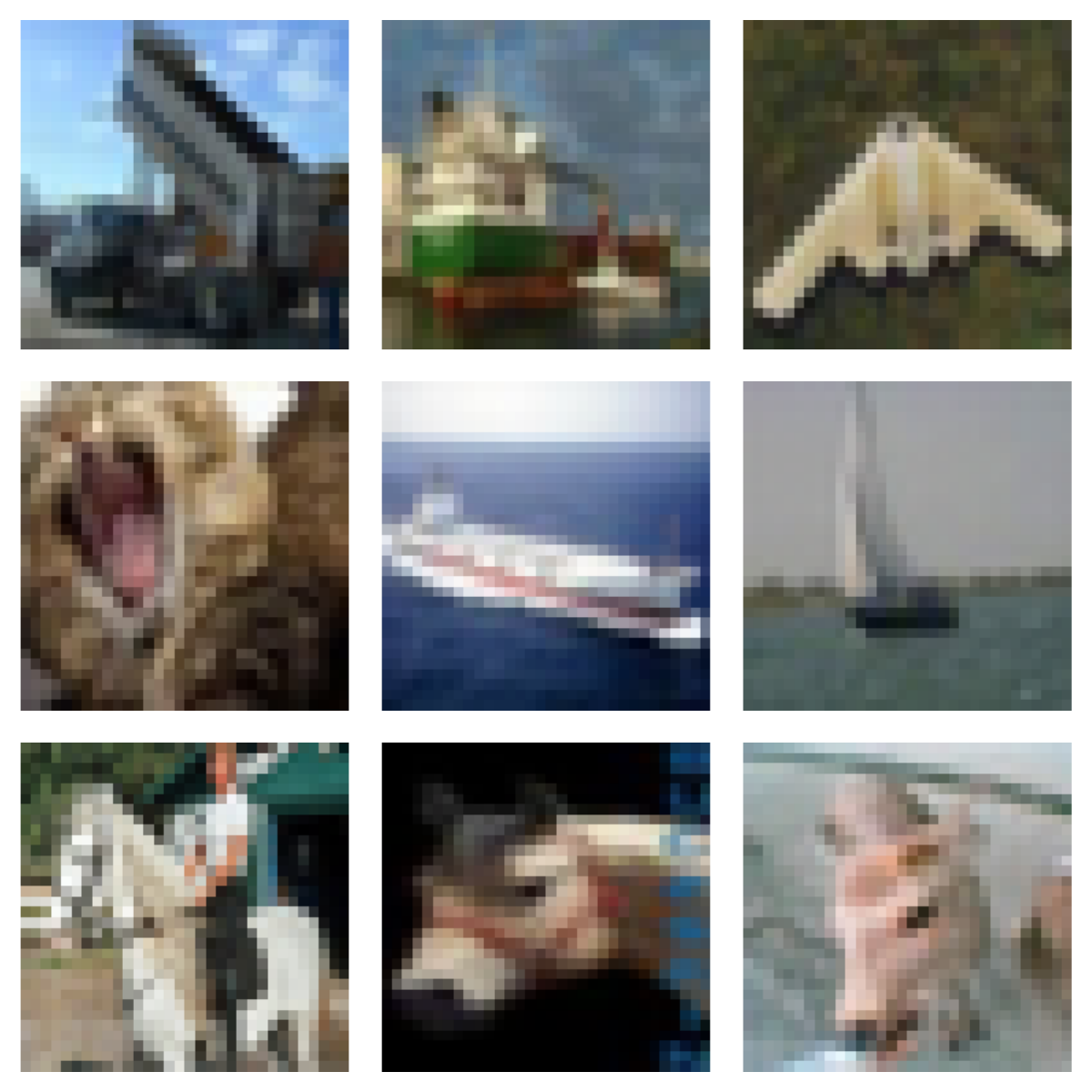}
        \caption{\shortstack{CIFAR10 Images;  \\ 224 Latent Dimensions}}
        \label{fig:cifar10_full}
    \end{subfigure}
    \hfill
    % Row for Configuration 1
    \begin{subfigure}[t]{0.15\textwidth}
        \centering
        \includegraphics[width=\textwidth]{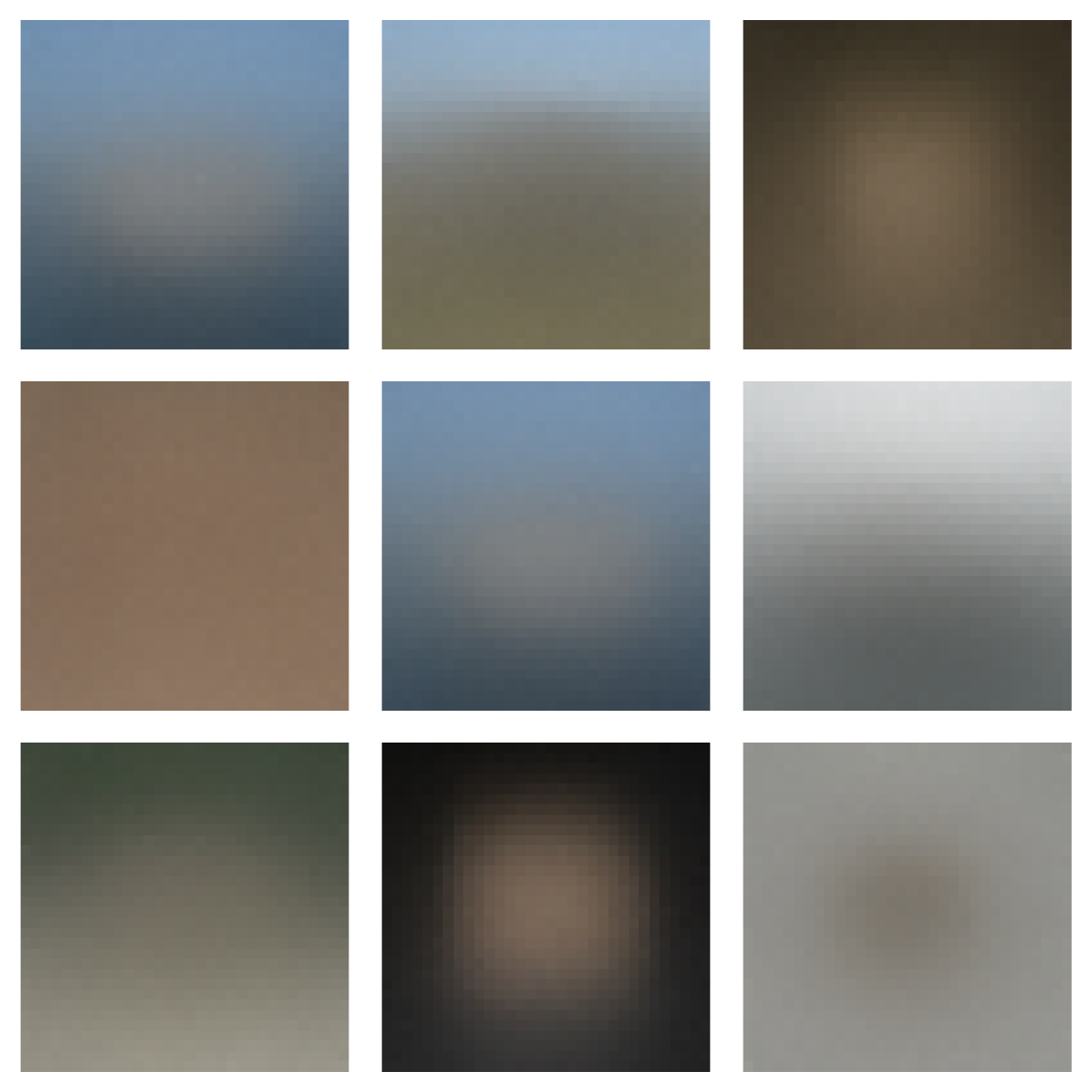}
        \caption{$r=0.1$  \shortstack{\\ 202 $\rightarrow$ 20 $\rightarrow$ 2}}
        \label{fig:cifar10_r_1}
    \end{subfigure}
    \hfill
    % Row for Configuration 2
    \begin{subfigure}[t]{0.15\textwidth}
        \centering
        \includegraphics[width=\textwidth]{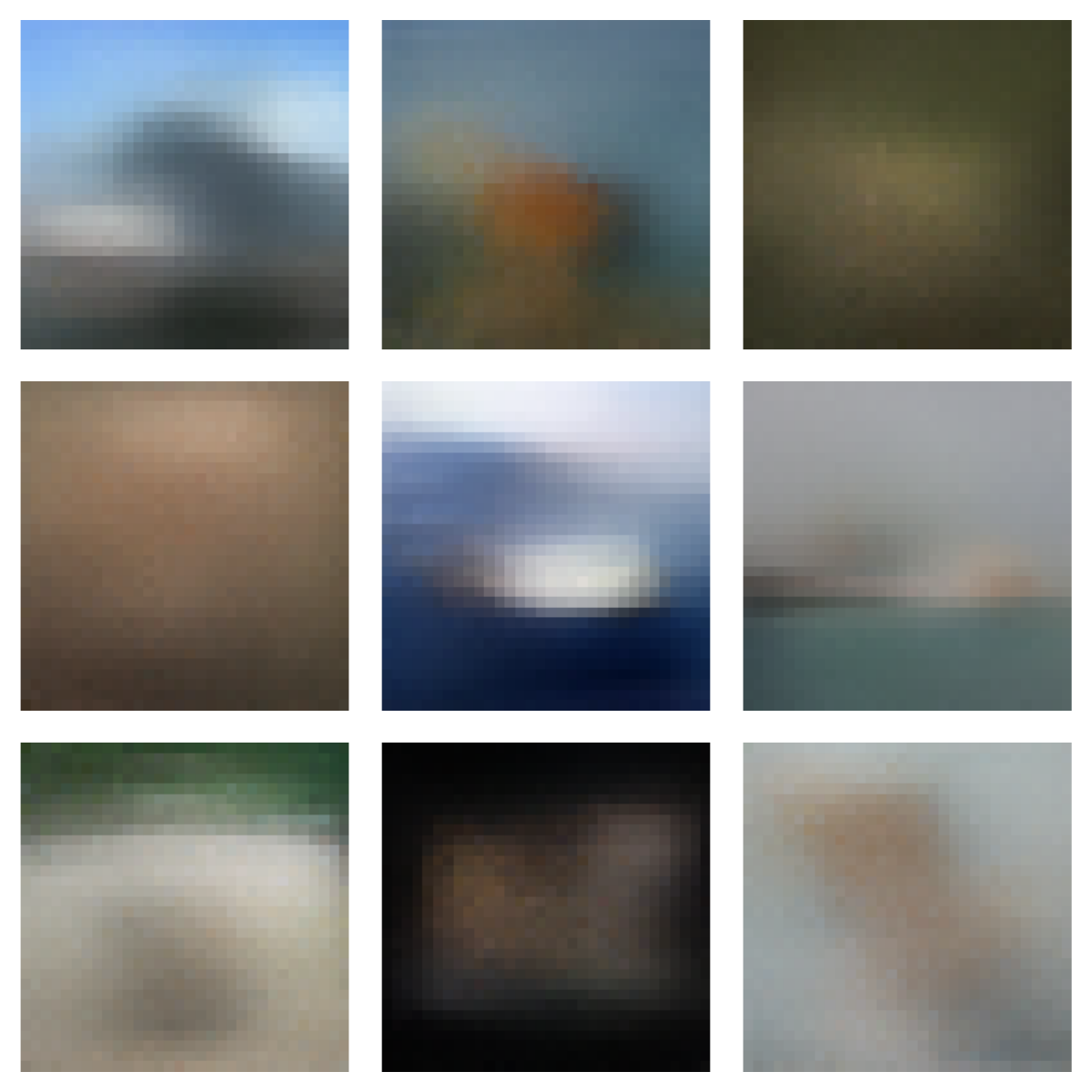}
        \caption{$r=0.25$  \shortstack{\\ 170 $\rightarrow$ 43 $\rightarrow$ 11}}
        \label{fig:cifar10_r_25}
    \end{subfigure}
    \hfill
    % Row for Configuration 3
    \begin{subfigure}[t]{0.15\textwidth}
        \centering
        \includegraphics[width=\textwidth]{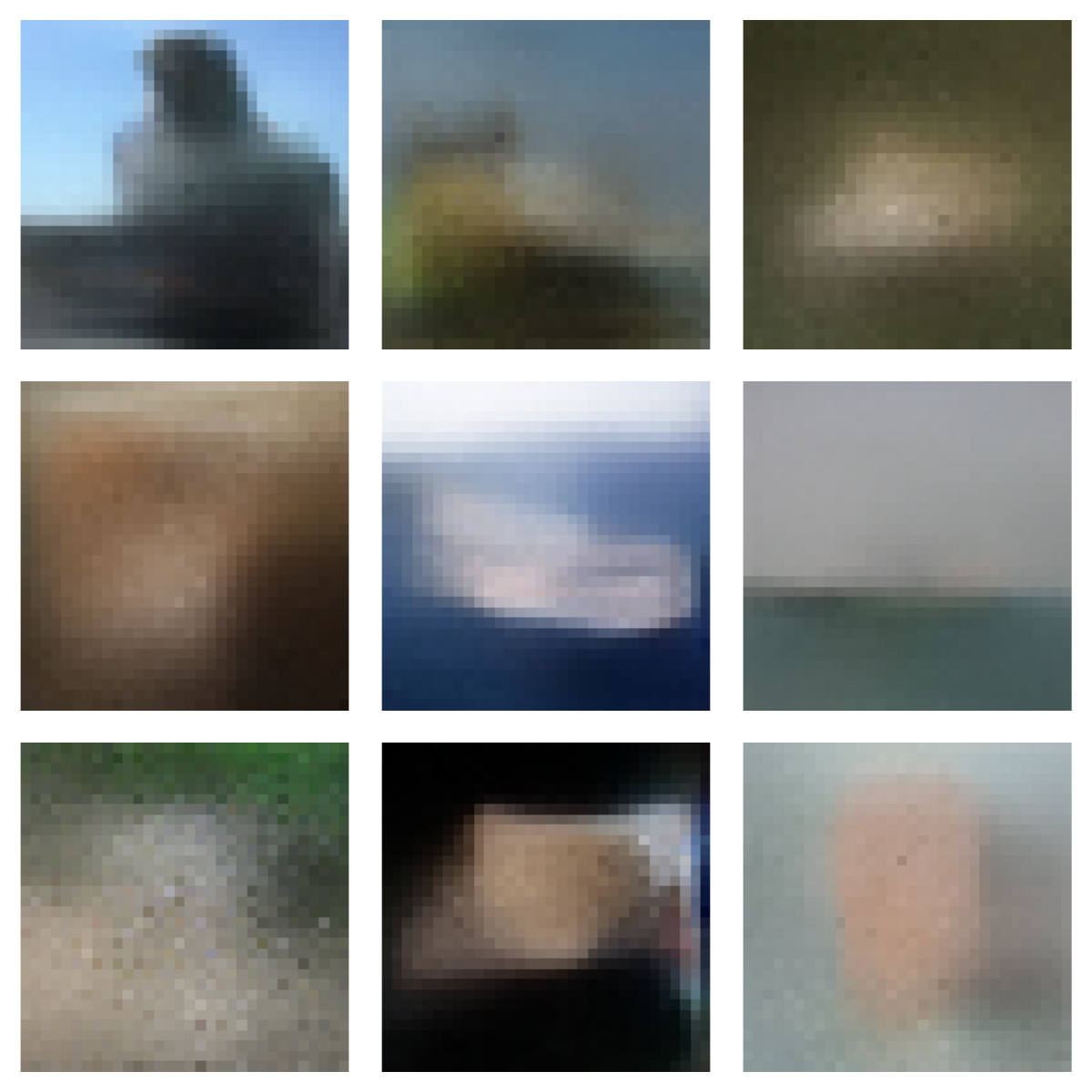}
        \caption{$r=0.5$  \shortstack{\\ 128 $\rightarrow$ 64 $\rightarrow$ 32}}
        \label{fig:cifar10_r_5}
    \end{subfigure}
    \hfill
    % Row for Configuration 4
    \begin{subfigure}[t]{0.15\textwidth}
        \centering
        \includegraphics[width=\textwidth]{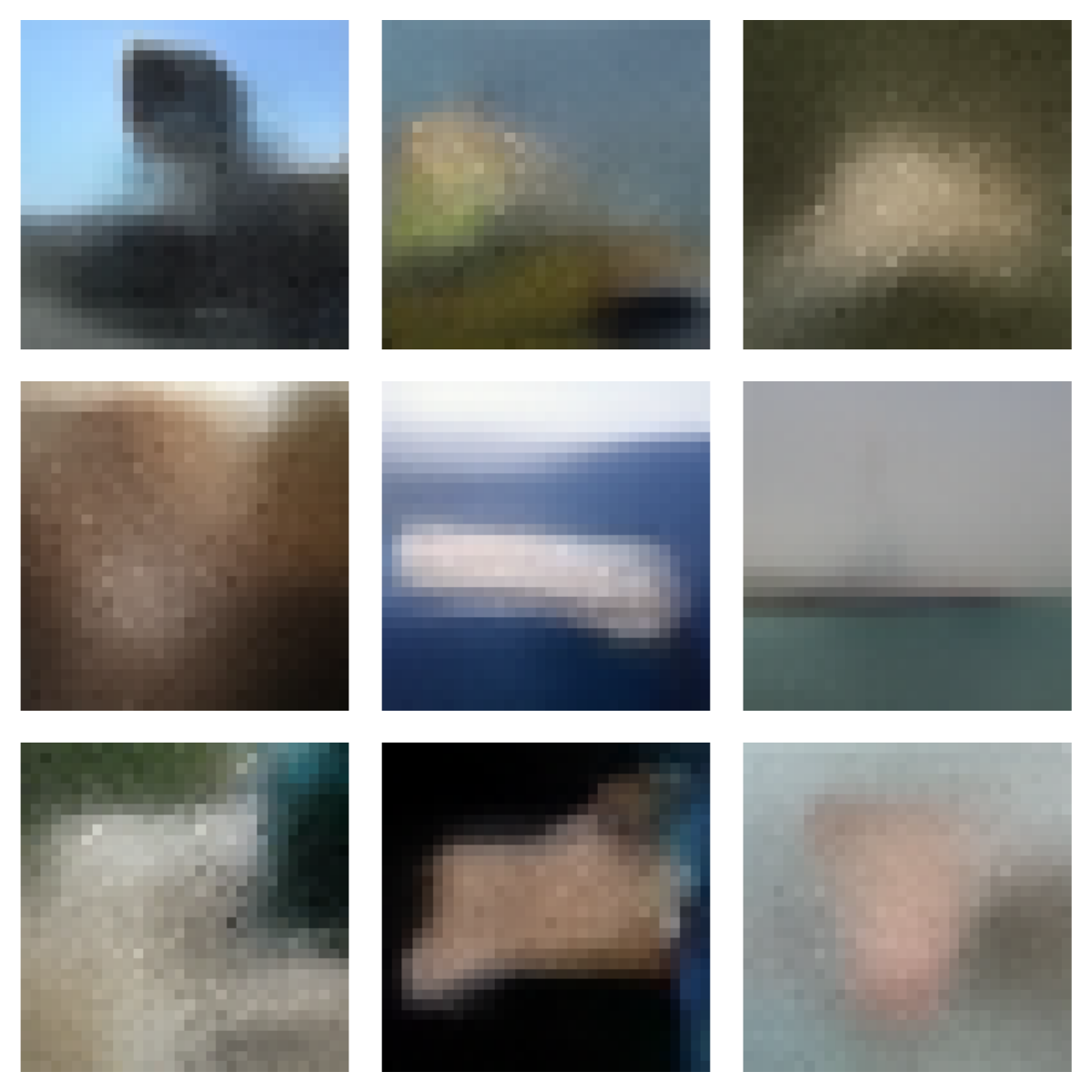}
        \caption{$r=0.75$  \shortstack{\\ 97 $\rightarrow$ 73 $\rightarrow$ 54}}
        \label{fig:cifar10_r_75}
    \end{subfigure}
    \caption{
        Reconstructed CIFAR10 samples under varying latent compression ratios $ r $. 
        At $ r = 0.1 $, the latent space is severely compressed $(202 \rightarrow 20 \rightarrow 2)$, and the reconstructions are highly uninformative. 
        As $ r $ increases to $ 0.5 $ and $ 0.75 $, the model allocates more bits to the latent hierarchy, enabling it to retain and reconstruct progressively more semantic detail. 
        This illustrates a core tradeoff: insufficient capacity leads to underrepresentation, while increasing capacity allows for richer reconstructions. 
        Only the top latent $ z_3 $ is inferred from the posterior; lower latents are sampled from the prior.
        }
    \label{fig:reconstructed_images}
\end{figure}

\section{Background and Related Work}
\label{sec:background}

\subsection{Variational Autoencoders}
\label{sec:vae}
\par Variational Autoencoders (VAEs) \cite{kingma2013vae} are deep generative models that learn to approximate complex data distributions by encoding inputs $\x$ into latent variables $\z$, which are then used to reconstruct the input. The training objective maximizes the Evidence Lower Bound (ELBO):

\begin{equation}
\label{eq:vae_elbo}
\mathcal{L} = \mathbb{E}_{\x \sim \mathcal{M}} \left[ \mathbb{E}_{q(\z \mid \x)} \left[ \log p_\theta(\x \mid \z) \right] - \text{KL}(q_\phi(\z \mid \x) \Vert p_\lambda(\z)) \right],
\end{equation}

\par where $\mathcal{M}$ is the empirical data distribution. The ELBO balances reconstruction accuracy with a KL penalty that regularizes the approximate posterior $q_\phi(\z \mid \x)$ toward a prior $p_\lambda(\z)$, promoting generalization and latent space regularity.

\subsection{Hierarchical Variational Autoencoders}
\label{sec:background_hvae}
Hierarchical VAEs (HVAEs) extend the VAE framework by introducing a hierarchy of latent variables $\z_1, \z_2, \dots, \z_L$, each representing different levels of abstraction \cite{sonderby2016laddervae, vahdat2021nvae}. Inference proceeds bottom-up:

\begin{equation}
    q_\phi(\z_1 \mid \x), \quad q_\phi(\z_2 \mid \x, \z_1), \quad \dots, \quad q_\phi(\z_L \mid \x, \z_{<L}).
\end{equation}

\par This structure enables multiscale encoding, where lower layers capture local features and higher layers encode abstract semantics. However, deeper latents often suffer from \textit{posterior collapse}, becoming statistically independent of the data. which limits the model’s representational power and effective capacity.

\subsection{OOD Detection with Deep Generative Models}
\label{sec:background_ood_dgm}
\par Deep generative models have been widely explored for OOD detection \cite{ren2019llrood, serr2021inputcomplexity, nalisnick2019dodeep, wang2020furtheranalysisood}. Although VAEs and related models can estimate data likelihoods, they often assign high likelihoods to OOD inputs. This phenomenon, attributed to \textit{estimation error} by \cite{zhang2021understandingoodfailures}, arises from architectural choices and objective mismatches rather than the generative framework itself.

\par To address these limitations, prior work has proposed modified training objectives that inject inductive biases \cite{bozkurt2021evaluatingcombvae} or impose latent consistency constraints \cite{dieng2021consistencyregularizationvae}. These approaches improve robustness but largely ignore structural aspects of the latent space. In contrast, we argue that \textit{latent dimension allocation}, how capacity is distributed across layers in HVAEs, is a critical but underexplored factor influencing estimation error.

\par Our work addresses two key gaps:
\begin{itemize}
    \item \textbf{Theoretical Gap}: The allocation of latent dimensions affects both expressivity and regularization. We formalize this tradeoff using the Information Bottleneck principle and derive bounds on generalization under varying allocation strategies.
    
    \item \textbf{Empirical Gap}: While prior work identifies HVAE likelihood-based scores as useful for OOD detection \cite{havtorn2021hvaes}, little has been done to systematically evaluate how latent structure impacts performance. We show that tuning latent allocation yields consistent gains across datasets.
\end{itemize}

\subsection{OOD Detection with Hierarchical VAEs}
\label{sec:background_ood_hvae}
Havtorn et al. \cite{havtorn2021hvaes} propose a likelihood ratio test to assess semantic in-distribution status using HVAEs. They compare:
\begin{enumerate}
    \item $\ELBO(x)$: the full ELBO where all latent variables are drawn from the posterior.
    \item $\ELBOK(x)$: a reduced ELBO where higher layers ($\z_{>k}$) are sampled from the prior.
\end{enumerate}

The difference, $\LLRK(x) = \ELBO(x) - \ELBOK(x)$, captures the informativeness of high-level latents; a large drop indicates that those latents carry distribution-specific signals, aiding in OOD detection.

However, prior HVAE work largely adopts fixed architectural patterns and does not explore the influence of latent dimensionality allocation. Our method explicitly optimizes this structure to enhance semantic abstraction and detection reliability.

\subsection{Hierarchical Structures and Abstraction}
\label{sec:hierarchy_abstract}
Deep neural networks rely on compositional depth to build increasingly abstract representations \cite{he2015resnet, krizhevsky2017imagenet, bengio2015deeplearning}. Early layers capture local patterns, while deeper layers encode semantically meaningful global features. Compression plays a key role: progressively reducing dimensionality emphasizes relevant features while discarding noise \cite{recanatesi2019dimensionality}.

This hierarchical compression is central to generalization and aligns with principles from physics, particularly Renormalization Group (RG) theory.

\subsubsection{Analogy: Renormalization Group Theory}
\label{sec:background_rgt}
RG theory describes how physical systems behave across scales by iteratively discarding irrelevant micro-level details while preserving critical structure \cite{wilson1971renorm, kadanoff1966scaling}. This process mirrors the abstraction in DNNs: successive latent transformations yield increasingly invariant and disentangled features \cite{mehta2014exactmapping}. We draw on this analogy to motivate the importance of controlled compression in hierarchical generative models.

\subsection{Information Bottleneck and Generalization}
\label{sec:background_dl_ibt}

The Information Bottleneck (IB) principle \cite{tishby2000infbottleneck} provides a formal framework for learning minimal sufficient representations. The IB objective seeks a stochastic encoding $Z$ of $X$ that maximally preserves information about a target $Y$, while discarding irrelevant input detail, as seen in ~\autoref{eq:ibp}:

\begin{equation}
\label{eq:ibp}
\inf_{p(z \mid x)} \left[ I(X; Z) - \beta I(Z; Y) \right],
\end{equation}

where $\beta$ governs the compression–prediction tradeoff. While latent dimensionality is not explicitly part of the IB objective, it plays a key role in controlling mutual information.

\par \cite{shamir2010learnandgenib} derive generalization bounds in ~\autoref{eq:ibt_gen_bounds} that make this dependency explicit:

\begin{equation}
\label{eq:ibt_gen_bounds}
I(Z; Y) \leq \hat{I}(Z; Y) + \mathcal{O} \left( \frac{|Z||\mathcal{Y}|}{\sqrt{n}} \right), \quad
I(X; Z) \leq \hat{I}(X; Z) + \mathcal{O} \left( \frac{|Z|}{\sqrt{n}} \right),
\end{equation}

\par where $|Z|$ is latent dimensionality, $|\mathcal{Y}|$ the number of classes, and $n$ the sample size. These bounds show that increasing latent capacity can improve expressivity, but may hurt generalization if it leads to overfitting or irrelevant feature capture.

\par Our work leverages the IB principle to formalize the tradeoff between compression (attenuation) and information loss in HVAEs. We derive a closed-form allocation strategy under a fixed latent budget and show that tuning this ratio significantly improves OOD detection.

\section{Approach}
\label{sec:approach}
\par We aim to explore how the relative dimensionality of the latents in an HVAE affects its performance in OOD detection.
To describe the relationship between different layers, we introduce the concept \emph{latent dimension ratio}.

\par To ensure a controlled comparison, we fix the total latent dimension budget, $b$ across all configurations, keeping all other parameters constant. By varying the ratios between layers within this fixed budget, we investigate how the relative latent dimension ratios of the configurations influence the model's ability to detect OOD data.

\subsection{Latent Dimension Ratio}
\label{sec:approach_latent_dimension_ratio}
We propose a systematic and constrained allocation strategy that enables precise control over the distribution of dimensions within the HVAE hierarchy. This approach allows for a rigorous examination of how latent configuration influences downstream performance, particularly in OOD detection.

To characterize the relationship between latent layers, we introduce the \textit{latent dimension ratio}, defined for any two adjacent layers $l_i$ and $l_{i+1}$ as:
\begin{equation}
    \label{eq:r_i}
    r_i = \frac{l_{i+1}}{l_i}.
\end{equation}

Under the simplifying assumption that this ratio remains constant across all layers, i.e., $\forall i,j: r_i \approx r_j = r$, we can fully parameterize the latent architecture using just $r$, the total latent dimensionality budget $b$, and the number of layers $N$.

\paragraph{Derivation of Layer Dimensionality}
Let $l_1$ denote the dimensionality of the first (highest-resolution) latent layer. The dimensionality of the $i$-th latent layer follows from a geometric progression:
\begin{equation}
    \label{eq:l_i}
        l_i = l_1 \cdot r^{i-1}.
\end{equation}
The total dimensionality budget across the hierarchy must satisfy:
\begin{equation}
    \label{eq:b_sum}
    b = \sum_{i=1}^{N} l_i = l_1 \cdot \left(1 + r + r^2 + \cdots + r^{N-1} \right).
\end{equation}
This is the finite geometric series:
\begin{equation}
    \label{eq:b_series}
    b = l_1 \cdot \frac{1 - r^N}{1 - r}, \quad \text{for } r \neq 1.
\end{equation}
Solving for $l_1$ gives:
\begin{equation}
     \label{eq:l1}
    l_1 = b \cdot \frac{1 - r}{1 - r^N}.
\end{equation}
Substituting back into the expression for $l_i$, we obtain a closed-form for any layer’s dimensionality:
\begin{equation}
    l_i = b \cdot \frac{(1 - r) \cdot r^{i-1}}{1 - r^N}.
    \label{eq:ratio_layer_dim}
\end{equation}

~\autoref{eq:ratio_layer_dim} is crucial for designing and analyzing HVAEs with a fixed latent capacity. It formalizes how representational capacity is compressed across the hierarchy and allows us to examine how changes in $r$ affect the informativeness of each latent layer.

\paragraph{Implementation Note.} In practice, we ensure that $\forall i,\, l_i \in \mathbb{Z}^{+}$ by rounding to the nearest integer. If rounding creates a mismatch with the budget $b$, we increment or decrement the highest layers (which are most compressive) to ensure that the total sum of dimensions matches $b$ while preserving hierarchical monotonicity.

\paragraph{Optimization Objective.}
We define a general objective function $\mathcal{F}(r)$ to quantify the effectiveness of an HVAE configuration:
\begin{equation}
    \label{eq_fr}
    \mathcal{F}(r) = \sum_{i=1}^N f(l_i),
\end{equation}
where $f(l_i)$ measures the contribution of the $i$-th layer’s dimensionality to OOD detection performance. Substituting ~\autoref{eq:ratio_layer_dim}, we have:
\begin{equation}
    \mathcal{F}(r) = \sum_{i=1}^N f\left(b \cdot \frac{(1 - r) \cdot r^{i-1}}{1 - r^N}\right).
\label{eq:objective_function}
\end{equation}
The optimal latent dimension ratio $r^*$ is then given by:
\begin{equation}
    r^* = \arg\max_{r \in (0,1]} \mathcal{F}(r).
\label{eq:r_star}
\end{equation}

\paragraph{Conclusion.}
The above derivation provides a theoretically grounded and numerically practical framework for allocating latent dimensions in HVAEs. ~\autoref{eq:ratio_layer_dim} enables us to systematically explore the effect of compression and abstraction across layers under a global constraint, while ~\autoref{eq:r_star} formulates the corresponding optimization problem. In subsequent sections, we connect this formulation to information-theoretic objectives and show that $r^*$ can be empirically determined and theoretically justified.

\subsection{Information-Theoretic Tradeoffs in Latent Dimensionality}
\label{sec:tradeoff_latent_dimension_ratio}
The mutual information between the input $X$ and the latent representation at a particular level $Z_i$ is closely tied to the dimensionality of the latent variable, $l_i = |Z_i|$. We denote this relationship by an implicit function $f(l_i)$, which captures the efficacy of the representation for OOD detection. Although $f$ is not known in closed form, it can be empirically evaluated by systematically varying $l_i$.

Poor choices of $l_i$ may lead to degraded performance due to either:

\begin{itemize}
    \item \textbf{Attenuation}: As $l_i$ increases, the marginal benefit of added capacity diminishes, i.e., $\frac{\partial f(l_i)}{\partial l_i} \rightarrow 0$.
    \item \textbf{Information Loss}: If $l_i \to 0$, the model lacks sufficient representational capacity and fails to encode meaningful semantics. In this regime, $f(l_i) \to \gamma$, where $\gamma$ is a task-specific lower bound.
\end{itemize}

These failure modes reflect the core tradeoff in the Information Bottleneck (IB) framework, which seeks to balance compression of the input $X$ with retention of task-relevant information for predicting $Y$. Let us analyze two limiting cases:

\paragraph{Case 1: $|Z| \to 0$ (Underfitting)}
\begin{equation}
    \label{eq:ibt_underfit}
    I(X; Z) \to 0, \quad I(Z; Y) \to 0, \quad O\left(\frac{|Z| |\mathcal{Y}|}{\sqrt{n}}\right) \to 0.
\end{equation}

In this regime, the latent representation is too constrained to capture relevant structure. We formalize this below:

\begin{proof}[Proof for Case 1: Mutual Information and Dimensionality]
\label{sec:mi_z}
Mutual information is given by:
$$
I(X; Z) = H(Z) - H(Z \mid X),
$$
where $H(Z)$ denotes the entropy of $Z$, and $H(Z \mid X)$ is the conditional entropy. For a Gaussian latent $Z \in \mathbb{R}^{|Z|}$ with covariance $\Sigma_Z$,
$$
H(Z) = \frac{1}{2} \log \left( (2\pi e)^{|Z|} \det \Sigma_Z \right).
$$
As $|Z| \to 0$, the space collapses, so $H(Z) \to 0$. Since $H(Z \mid X) \leq H(Z)$, it follows that:
$$
H(Z \mid X) \to 0 \quad \Rightarrow \quad I(X; Z) = H(Z) - H(Z \mid X) \to 0.
$$

Similarly, if $Z$ lacks expressive power, its relevance to $Y$ vanishes: $I(Z; Y) \to 0$. Although the complexity term $O\left(\frac{|Z| |\mathcal{Y}|}{\sqrt{n}}\right)$ improves, the representation is trivial and ineffective for prediction. 
\end{proof}

\paragraph{Case 2: $|Z| \to \infty$ (Overfitting)}
\begin{equation}
    \label{eq:ibt_overfit}
    I(X; Z) \to H(X), \quad I(Z; Y) \to H(Y), \quad O\left(\frac{|Z| |\mathcal{Y}|}{\sqrt{n}}\right) \to \infty.\footnotemark[1]
\end{equation}

As dimensionality grows unbounded, the latent encodes nearly all information from $X$, including irrelevant noise. While the mutual information terms maximize, generalization degrades due to the high model complexity. In this regime, the IB principle is violated due to lack of compression.

\begin{proof}[Proof for Case 2: Mutual Information and Dimensionality]
    Mutual information is again given by:
    $$
    I(X; Z) = H(Z) - H(Z \mid X).
    $$
    
    If $ |Z| \to \infty $, and $ Z \in \mathbb{R}^{|Z|} $ is Gaussian with covariance $ \Sigma_Z $, then the entropy:
    $$
    H(Z) = \frac{1}{2} \log \left( (2\pi e)^{|Z|} \det \Sigma_Z \right) \to \infty.
    $$
    
    Assuming the encoder becomes nearly deterministic as capacity increases (i.e., $ H(Z \mid X) \to 0 $), it follows that:
    $$
    I(X; Z) = H(Z) - H(Z \mid X) \to \infty.
    $$
    
    In this regime, $ Z $ retains all information about $ X $, including noise and irrelevant features. Although $ I(Z; Y) \to H(Y) $ may also increase, the mutual information estimate becomes unstable due to high complexity.
    
    From the generalization bound:
    $$
    O\left(\frac{|Z| |\mathcal{Y}|}{\sqrt{n}}\right) \to \infty,
    $$
    which indicates increased risk of overfitting and degraded generalization.
    
    Thus, unbounded latent capacity leads to representations that are overly complex and semantically inefficient.
\end{proof}

\footnotetext[1]{Entropy represents upper bound on mutual information.}

Together, these observations reinforce the need for a principled tradeoff: one must choose latent dimensionalities that retain semantically useful information while avoiding redundancy. We leverage this analysis to motivate our search for an optimal latent allocation strategy in ~\autoref{sec:optimal_latent_dimension_ratio}

\subsection{Optimal Latent Compression Ratio}
\label{sec:optimal_latent_dimension_ratio}

The tradeoff between compression and predictive utility, formalized by the Information Bottleneck (IB) principle, has been explored in both theoretical and empirical studies. \cite{alemi2016deepvib} analyze this tradeoff in models trained using a variational IB formulation, while \cite{ruff2019deepssad} empirically show that increasing the latent dimensionality $|Z|$ generally improves out-of-distribution (OOD) detection, up to a point of diminishing returns. These findings suggest that a model must allocate sufficient capacity to avoid underfitting, but not so much that it overfits to irrelevant detail. This motivates the search for an optimal allocation strategy across hierarchical latent layers.

Let $b$ denote the total latent dimensionality budget of an HVAE with $N$ layers. Assume that each layer $i$ receives $l_i$ dimensions and that dimensionality is distributed geometrically across the hierarchy with a fixed compression ratio $r$, such that:

\begin{equation}
  l_i = b \cdot \frac{(1 - r) \cdot r^{i-1}}{1 - r^N}.
  \label{eq:optimal_latent_allocation}
\end{equation}

Let $f(l_i)$ denote the latent utility function, implicitly relating dimensionality to detection efficacy. The goal is to find the compression ratio $r^\ast$ that maximizes total efficacy under the latent budget $b$:
\begin{equation}
    r^\ast = \underset{r \in (0, 1]}{\argmax} \sum_{i=1}^{N} f(l_i) = \underset{r \in (0, 1]}{\argmax} \sum_{i=1}^{N} f\left(b \cdot \frac{(1 - r) \cdot r^{i-1}}{1 - r^N}\right).
    \label{eq:optimal_ratio}
\end{equation}

\begin{prop}
\label{prop:existence_r_star}
An optimal latent compression ratio $r^\ast$ exists that maximizes total efficacy $\mathcal{F}(r)$ under a fixed budget $b$.
\end{prop}

\begin{proof}
Define the objective:
\begin{equation}
\mathcal{F}(r) = \sum_{i=1}^N f\left(b \cdot \frac{(1 - r) \cdot r^{i-1}}{1 - r^N}\right).
\end{equation}

We assume the following conditions on $f$:
\begin{itemize}
    \item \textbf{Continuity:} $f$ is continuous in $l_i$.
    \item \textbf{Saturation:} For large $l_i$, $\frac{\partial f(l_i)}{\partial l_i} \leq 0$, indicating diminishing returns.
    \item \textbf{Information Loss:} $\lim_{l_i \to 0} f(l_i) = \gamma$ for some $\gamma$, representing a minimum threshold for utility.
\end{itemize}

Since $r \in (0, 1]$ and $f$ is continuous, each term in $\mathcal{F}(r)$ is continuous in $r$. Thus, $\mathcal{F}(r)$ is a continuous function on a compact domain $(0, 1]$. Additionally, because $f(l_i) \leq M$ for some bound $M$, we have:
\begin{equation}
    \mathcal{F}(r) \leq N \cdot M,
\end{equation}
ensuring boundedness.

By the Extreme Value Theorem, a continuous function on a compact interval attains its maximum. Therefore, there exists $r^\ast \in (0, 1]$ such that:
\begin{equation}
    \mathcal{F}(r^\ast) = \max_{r \in (0, 1]} \mathcal{F}(r).
\end{equation}
\end{proof}

This theoretical guarantee validates our formulation: under mild regularity conditions on $f$, an optimal ratio $r^\ast$ exists that balances capacity across the hierarchy to maximize downstream utility. 

In the context of the Information Bottleneck, such a configuration ensures that each layer in the HVAE retains just enough information to support the task (e.g., OOD detection), without wasting capacity or overfitting to spurious details. Intuitively, this corresponds to a form of progressive abstraction, where higher layers compress while preserving semantic structure.

Hence, a hierarchy satisfying $\forall i, r_i = r^\ast$ defines an optimal allocation strategy. In cases where layer-wise ratios vary, i.e., $\exists i, j$ such that $r_i \ne r_j$, we explore alternate allocations empirically in ~\autoref{sec:results}.

\subsection{Empirical Evaluation of \texorpdfstring{$r^\ast$}{r*}}
\label{sec:empirical_eval}

While the existence of an optimal compression ratio $r^\ast$ is guaranteed theoretically, its exact value depends on the specific dataset, model, and downstream task. As the latent utility function $f$ is not known in closed form, we determine $r^\ast$ empirically.

We perform a grid search over candidate values of $r \in (0,1]$, subject to a fixed latent budget $b$ and hierarchy depth $N$. For each configuration, we evaluate OOD detection performance using standard metrics: AUROC, AUPRC, FPR80, and FPR95. In addition, we estimate the mutual information $I(X; Z_i)$ at each latent level $Z_i$ to better understand how changes in $r$ affect semantic compression and information retention.

We define the empirical optimum $r^\ast$ as the value that yields the best tradeoff between information retention and generalization across a range of OOD detection tasks. Experimental results and full evaluation details are presented in ~\autoref{sec:results}.

To support this analysis, we describe the estimator used to compute $I(X; Z_i)$ for hierarchical VAEs below.

\subsection{Mutual Information Estimator}
\label{proof:estimator_derivation}
For a given latent in the hierarchy $Z_i$, the mutual information with the input $X$ is defined as:
\begin{equation}
    \label{eq:mutual_information}
    I(X; Z_i) = D_{\text{KL}} \left( p(x, z_i) \,\middle\|\, p(x) \cdot p(z_i) \right)
\end{equation}

Using the definition of KL divergence, this can be expressed as an expectation over the joint distribution:
\[
     I(X; Z_i) = \mathbb{E}_{p(x, z_i)} \left[ \log \frac{p(x, z_i)}{p(x) \cdot p(z_i)} \right]
\]

Applying logarithmic properties, we expand this into three terms:
\begin{align}
     I(X; Z_i) &=  \mathbb{E}_{p(x, z_i)} \left[ \log p(x, z_i) \right] \label{eq:mi_joint} \\
     &\quad - \mathbb{E}_{p(x, z_i)} \left[ \log p(x) \right] \label{eq:mi_marginal} \\
     &\quad - \mathbb{E}_{p(x, z_i)} \left[ \log p(z_i) \right] \label{eq:mi_prior}
\end{align}

These terms are approximated as follows:

\begin{description}
    \item[\autoref{eq:mi_joint}] The log-likelihood of the joint distribution is computed by sampling $z_i$ from the approximate posterior $q(z_i \mid x)$, and sampling $z_{j \ne i}$ from their respective priors. The log-likelihood of $x$ is evaluated under the resulting generative path.
    
    \item[\autoref{eq:mi_marginal}] The marginal likelihood $\log p(x)$ is intractable and approximated using the Evidence Lower Bound (ELBO). Since the non-target latents are sampled from their priors, the KL term vanishes at layer $i$.
    
    \item[\autoref{eq:mi_prior}] The log-prior over $z_i$ is evaluated directly using the prior distribution (e.g., standard Gaussian). The expectation is taken over the samples $z_i \sim q(z_i \mid x)$.
\end{description}

This estimator allows us to quantify how effectively each latent encodes information about the input under different compression ratios $r$, linking representation capacity to downstream task utility.

\section{Research Questions and Experiment Design}
\label{sec:rq}

\subsection{Research Questions}

To evaluate the influence of latent dimension allocation on the performance of HVAE-based OOD detectors, we investigate the following research questions:

\begin{description}
\item[\textbf{RQ1}] \textit{How does the latent dimension ratio $r$ affect the OOD detection performance of HVAE models across different datasets?}
\item[\textbf{RQ2}] \textit{What compression ratio $r^\ast$ yields optimal OOD detection performance under a fixed latent budget, and how consistent is this optimum across datasets?}
\item[\textbf{RQ3}] \textit{How does mutual information $I(X; Z_i)$ vary with $r$, and can it explain observed trends in OOD performance?}
\end{description}

\subsection{Experimental Setup}
To explore these questions, we conduct a comprehensive series of experiments using both grayscale and natural image datasets. For each configuration, we control for total latent budget $b$ and depth $N$, and systematically vary the latent compression ratio $r$. We include the ID/OOD dataset pairs from prior work \cite{havtorn2021hvaes}, and extend the evaluation to additional ID/OOD pairings.

\paragraph{Datasets.}
To evaluate the generality and robustness of our method, we conduct experiments on both grayscale and natural image datasets. Following \cite{havtorn2021hvaes}, we include the standard evaluation settings: FashionMNIST \cite{xiao2017fmnist} (ID) vs.\ MNIST \cite{deng2012mnist} (OOD), and CIFAR10 \cite{krzhevsky2009cifar10} (ID) vs.\ SVHN \cite{netzer2011svhn} (OOD). \emph{We expand this evaluation in two key ways.}

First, we increase OOD diversity for FashionMNIST by adding KMNIST and notMNIST as additional OOD datasets. Second, we introduce a new ID setting using MNIST, evaluating it against three OOD counterparts: FashionMNIST, KMNIST, and notMNIST. Finally, we augment the natural image setting by including CelebA as an additional OOD dataset alongside SVHN. 

This results in a total of three ID datasets—FashionMNIST, MNIST, and CIFAR10, each paired with at least 2 OOD datasets, enabling a more comprehensive assessment of our framework across distinct visual domains and difficulty levels.

\subsection{Treatments for RQ1–RQ3}

We explore compressed configurations by varying the ratio $r$ across a predefined grid:
$$
r \in \{0.1, 0.25, 0.5, 0.75\},
$$
which spans the range from aggressive compression to minimal compression.
For each $r$, we compute the corresponding layer-wise dimensionalities using ~\autoref{eq:ratio_layer_dim}, train the HVAE model, and evaluate OOD performance using standard detection metrics.

To answer \textbf{RQ3}, we estimate the mutual information $I(X; Z_3)$ \footnotemark[2] for each compressed configuration (i.e., same $r$ sweep), across all ID datasets. This ensures consistent experimental conditions for comparing OOD performance with latent information content.

For settings directly comparable to \cite{havtorn2021hvaes}, we adopt their latent budget and depth:
\begin{itemize}
    \item FashionMNIST (ID), MNIST (ID): $b = 32$, $N = 3$
    \item CIFAR10 (ID): $b = 228$, $N = 3$
\end{itemize}

\subsection{Controls for RQ1}
\label{sec:rq_control}
To further understand the role of latent configuration beyond fixed $r$, we include several non-compressed control configurations. These span a variety of expansion, compression, and stable allocation regimes:

\begin{description}
    \item[Expand-then-Retain:] \(6 \to 13 \to 13\) (Grayscale), \(42 \to 91 \to 91\) (Natural Images)
    \item[Progressive Expansion:] \(8 \to 10 \to 14\), \(56 \to 80 \to 98\)
    \item[Expand-then-Compress:] \(8 \to 16 \to 8\), \(56 \to 112 \to 56\)
    \item[Stable Allocation:] \(10 \to 11 \to 11\), \(70 \to 77 \to 77\)
    \item[Compress-then-Expand:] \(13 \to 6 \to 13\), \(91 \to 42 \to 91\)
\end{description}

\paragraph{Control Configuration Justification}These control groups were selected to reflect common heuristics in architecture design and ablations used in prior work. For example, expand–then–compress mimics encoder-decoder symmetry, stable allocations test the effect of uniform latent capacity, and compress–then–expand configurations simulate bottleneck-style representations. Together, these serve as diverse, plausible baselines to assess whether our geometric successive compression assumption is actually necessary for strong OOD performance, or whether arbitrary tuning can suffice.

\subsection{Procedure}
For each configuration (compressed or control), the HVAE is trained from scratch. For compressed settings, layer dimensionalities are determined using ~\autoref{eq:ratio_layer_dim}. We then compute the $\text{LLR}^{2}$ score for each configuration and evaluate OOD detection performance using the following metrics:

\subsection{Evaluation Metrics}
To answer RQ1 and RQ2, we use the metrics employed in prior work on HVAE OOD Detection : \textbf{FPR@80},  \textbf{FPR@95},  \textbf{AUROC} and\textbf{ AUPRC} \cite{havtorn2021hvaes, li2022adaptiveratio}. RQ3 results are derived from mutual information estimates obtained across configurations.\footnotetext[2]{Using estimator described in ~\autoref{proof:estimator_derivation}}
\section{Experimental Results}
\label{sec:results}

We organize our findings around the research questions introduced in ~\autoref{sec:rq}.

\subsection{RQ1: How Does Latent Compression Ratio Affect OOD Performance?}
\label{sec:rq1-results}

We evaluate OOD detection performance across a sweep of compression ratios $r \in \{0.1, 0.25, 0.5, 0.75\}$ using standard metrics (AUROC, AUPRC, FPR80, FPR95). ~\autoref{fig:heatmap_fmnist}, \autoref{fig:heatmap_mnist}, and \autoref{fig:heatmap_cifar} present the results for grayscale and natural image datasets, respectively.

\paragraph{Interpreting the Heatmaps.}
Each heatmap visualizes detection performance across configurations (rows) and metrics (columns), grouped by OOD dataset. \textbf{Red denotes better performance}, \textbf{blue indicates worse}.
We note that:
\begin{itemize}
    \item Higher is better for AUROC and AUPRC.
    \item Lower is better for FPR80 and FPR95.
\end{itemize}
To compute the mean score for each configuration (rightmost column), we average the values of AUROC, AUPRC, and $(1 - \text{FPR})$ for each metric, so that all four metrics are aligned in scale where higher is better. This normalized mean provides a single summary value for comparing overall detection quality across configurations.
Uniformly red rows indicate strong and stable performance across metrics, providing a visual cue for identifying the most effective latent compression strategies.

\paragraph{FashionMNIST In-Distribution.}
In the compressed setting, performance increases monotonically with $r$ until reaching a maximum at $r = 0.75$, which is the empirically optimal setting across all OOD datasets considered. The lowest-performing configuration corresponds to $r = 0.1$, indicating that excessive compression consistently degrades detection accuracy.

\paragraph{MNIST In-Distribution.}
In contrast, MNIST exhibits optimal performance at a moderate compression level of $r = 0.5$. Performance degrades beyond this point, suggesting diminishing returns and possible overfitting at $r = 0.75$.

\begin{figure*}[ht]
    \centering
    \includegraphics[width=\linewidth]{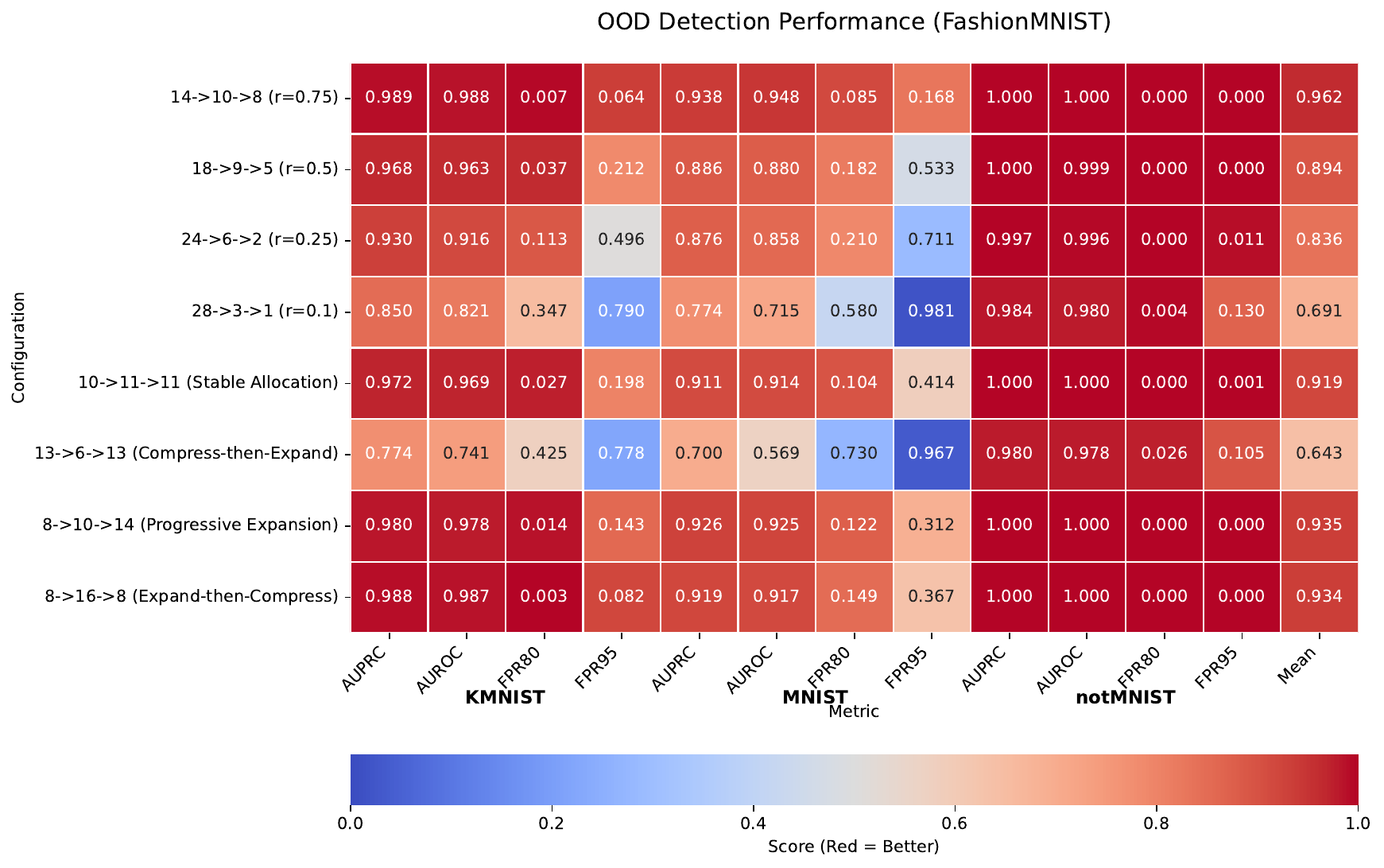}
     \caption{The optimal configuration for FashionMNIST-In is $r^{\ast} = 0.75$, corresponding to $\{ 14 \rightarrow 10 \rightarrow 8 \}$. This configuration is optimal across all metrics for FashionMNIST-In vs all OOD pairings.}
    \label{fig:heatmap_fmnist}
\end{figure*}

\begin{figure*}[ht]
    \centering
    \includegraphics[width=\linewidth]{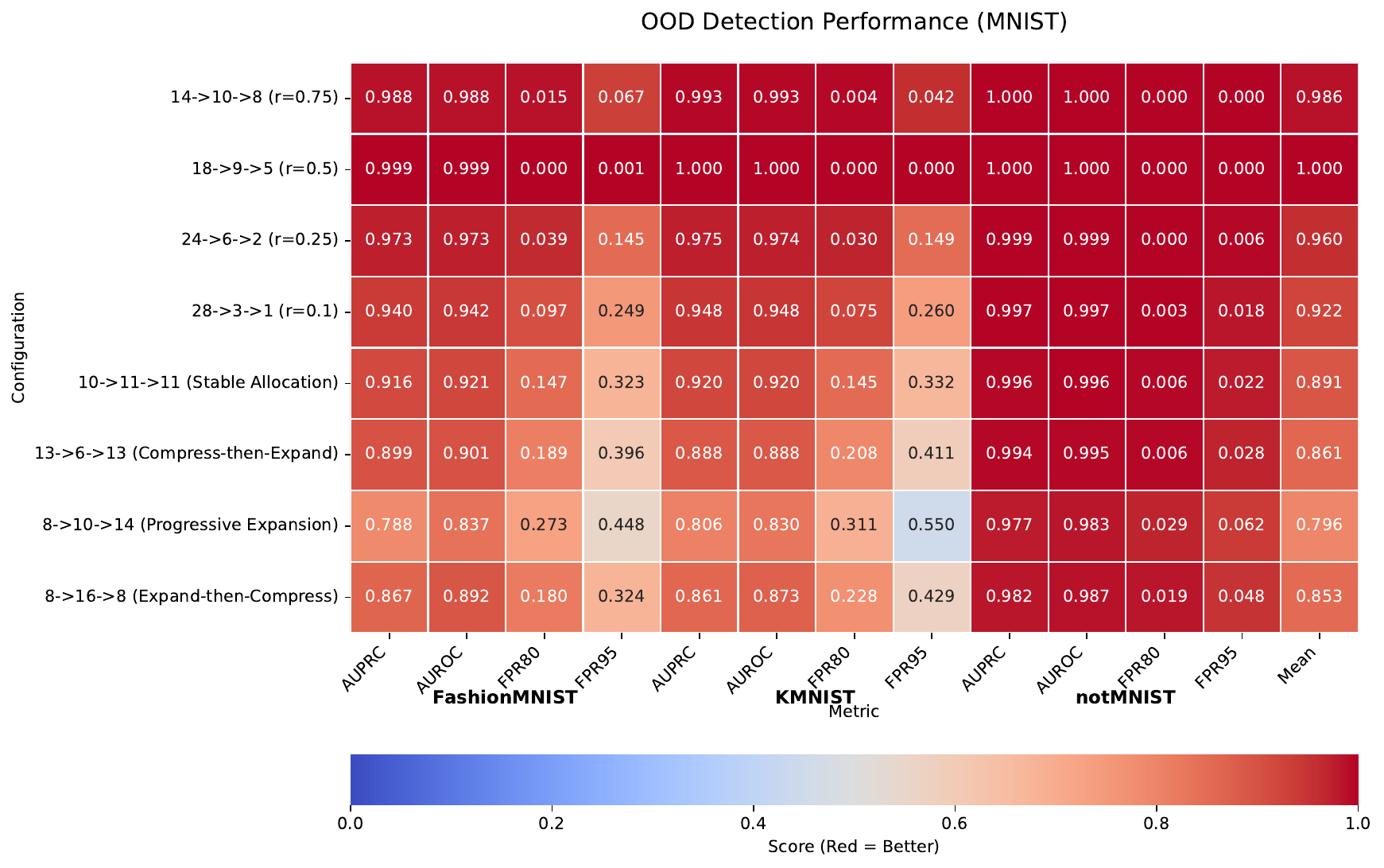}
     \caption{The optimal configuration for MNIST-In is $r^{\ast} = 0.5$, corresponding to $\{ 18 \rightarrow 9 \rightarrow 5 \}$. This configuration is optimal across all metrics for MNIST-In vs all OOD pairings. Compressed configurations drastically outperform controlled ones.}
    \label{fig:heatmap_mnist}
\end{figure*}

\paragraph{CIFAR10 In-Distribution.}
For natural images, the best-performing configuration is found at $r = 0.25$, suggesting that shallower hierarchies benefit from moderate compression. Both excessive compression ($r = 0.1$) and minimal compression ($r = 0.75$) yield sub-optimal results, reinforcing the importance of balance.

\paragraph{Comparison to Prior Work.}
Across all datasets, our method outperforms the baseline HVAE configurations used in \cite{havtorn2021hvaes}, (\textit{expand then compress} for grayscale and \textit{r=0.5} for natural images), which rely on fixed latent structures rather than optimized ratios.  The $r$-sweep consistently identifies configurations with superior AUROC, AUPRC, and lower FPR values, demonstrating that latent compression ratio is a critical and previously underexplored design axis for HVAE-based OOD detection.

These results confirm a consistent, nonlinear relationship between $r$ and detection efficacy, validating the motivation for our geometric latent allocation framework.

\paragraph{Control Configuration Performance.}
To contextualize the impact of our geometric allocation strategy, we compared it against several non-compressed control configurations described in ~\autoref{sec:rq_control}. Across all datasets, the best-performing compressed configurations consistently outperform these controls. For example, the compress-then-expand configuration ($13 \to 6 \to 13$) yields significantly higher FPR and lower AUPRC on FashionMNIST and MNIST, indicating poor OOD discrimination. Stable allocations such as $10 \to 11 \to 11$ exhibit more balanced behavior, but still fail to match the sensitivity and generalization of optimized $r$ configurations. These results support our central claim: principled dimension allocation, rather than arbitrary tuning, yields robust, high-performing representations for OOD detection.

\begin{figure*}[ht]
    \centering
    \includegraphics[width=0.85\linewidth]{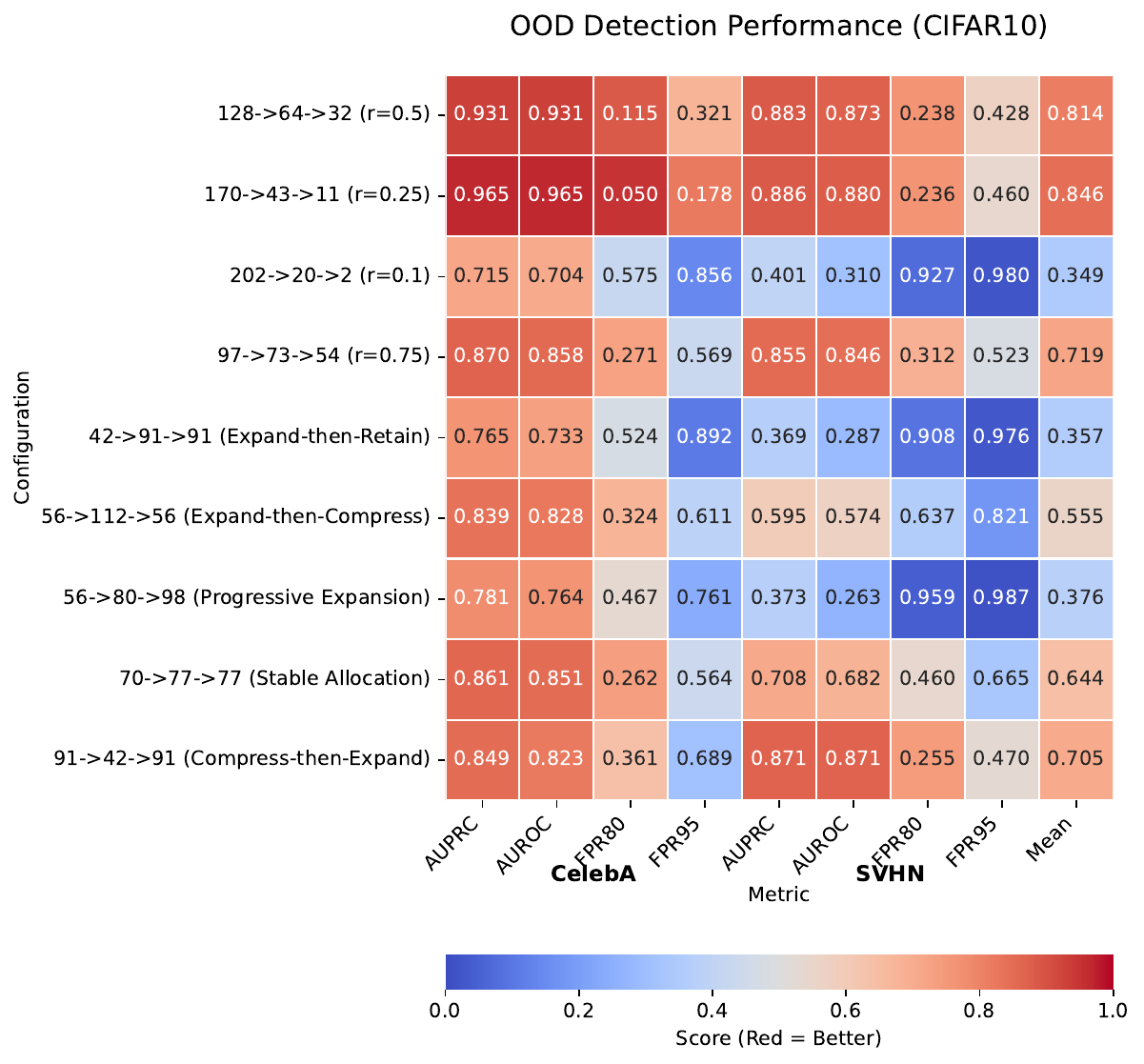}
    \caption{The configuration, $r^{\ast}=0.25$ returns the optimal configuration of $170 \rightarrow 43 \rightarrow 11$. This configuration is optimal across all metrics for CIFAR10 for all pairings except one. This is in the case of FPR95 for SVHN. It is worth noting that for CelebA out, the compressed configurations drastically outperform their controlled counterparts. }
    \label{fig:heatmap_cifar}
\end{figure*}

\subsection{RQ2: What Compression Ratio \texorpdfstring{$r^\ast$}{r*} is Optimal Across Datasets?}
\label{sec:rq2-results}

We define $r^\ast$ as the compression ratio that yields the highest average performance across OOD pairings for a given ID dataset. Our experiments identify distinct dataset-specific optima:

\begin{itemize}
    \item \textbf{FashionMNIST:} $r^\ast = 0.75$ across all ID-OOD pairings.
    \item \textbf{MNIST:} $r^\ast = 0.5$, indicating optimal performance at intermediate compression.
    \item \textbf{CIFAR10:} $r^\ast = 0.25$, suggesting early abstraction is beneficial for natural images.
\end{itemize}

These findings demonstrate that while a clear optimum exists for each dataset, $r^\ast$ is not universal. Instead, it reflects dataset-specific representational demands, consistent with our formulation of the latent utility function $f(l_i)$.

\subsection{RQ3: How Does Mutual Information \texorpdfstring{$I(X; Z_i)$}{I(X; Zi)} Vary with \texorpdfstring{$r$}{r}?}
\label{sec:rq3-results}

To understand the information-theoretic behavior underlying these trends, we estimate the mutual information $I(X; Z_3)$ for the final latent level across all datasets and compression ratios. ~\autoref{fig:mi_plot} shows the results.

\paragraph{R vs Mutual Information in ~\autoref{fig:mi_plot}.}
The x-axis denotes the latent dimension ratio $r$, and the y-axis reports the estimated mutual information $I(X; Z_3)$ for each dataset. The red crosses indicate the empirically optimal $r^\ast$ for each dataset as identified in ~\autoref{sec:rq2-results}. The shaded regions capture variance for the MI estimates across runs. A desirable configuration achieves high utility (i.e., performance, as seen in ~\autoref{fig:heatmap_fmnist}, ~\autoref{fig:heatmap_mnist}, ~\autoref{fig:heatmap_cifar}) without unnecessarily high $I(X; Z_3)$, indicating a semantically efficient representation.

\begin{figure}[ht]
    \centering
    % Adjust the width parameter to make the plot smaller
    \includegraphics[width=0.75\linewidth]{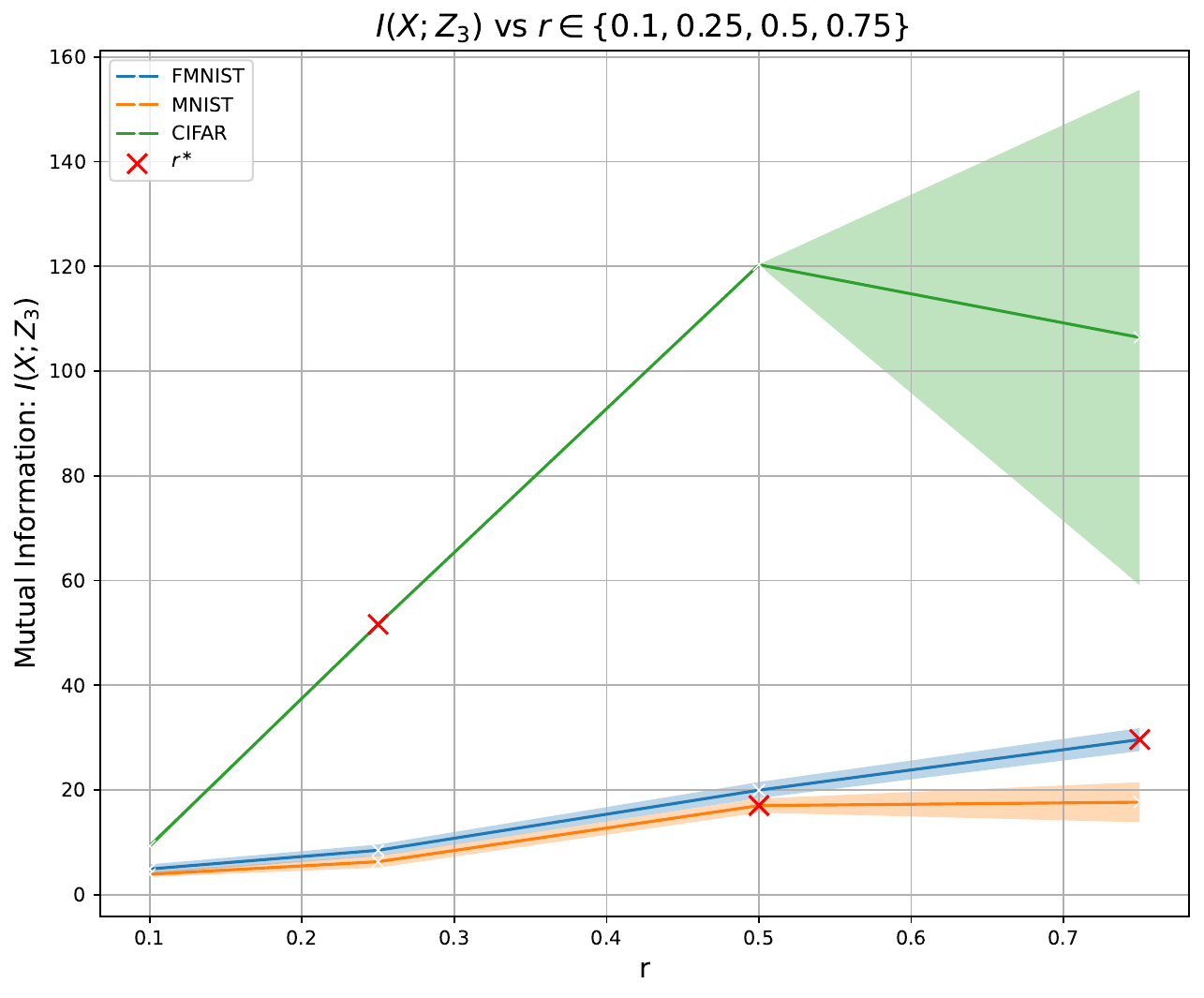}
    \caption{Mutual Information as r is varied in compressed configurations. Here we can observe the tradeoff between information loss and attenuation. We visualize the mutual information estimates for samples taken from each respective dataset, with the latent $Z_{3}$ for each configuration corresponding to r. $r^{\ast}$ for each configuration is higlighted.}
    \label{fig:mi_plot}
\end{figure}

\paragraph{Information Loss}
At $r = 0.1$, all datasets exhibit low mutual information, consistent with aggressive compression. These configurations also correspond to the weakest detection performance, validating the theoretical link between under-representation and task failure.

\paragraph{Attenuation}
As $r \to 1$, mutual information increases, but detection performance plateaus or declines. For example, although CIFAR10’s $I(X; Z_3)$ increases sharply with $r$, its optimal $r^\ast = 0.25$ lies far below the MI maximum, indicating diminishing returns and likely overfitting.

\paragraph{Variance Trends}
The shaded region around the CIFAR10 curve grows with $r$, showing that $\Var(I(X; Z_3))$ increases significantly with latent dimensionality. This suggests that larger latent spaces are more sensitive to noise and dataset artifacts, reinforcing the importance of structural constraints.

Together, these results support the tradeoff predicted by the Information Bottleneck: an optimal compression ratio $r^\ast$ exists that balances semantic sufficiency with generalization robustness.

\section{Analysis}
\label{sec:analysis}

\subsection{Optimal Compression and Generalization}
\label{sec:analysis_optimal}

Across all configurations evaluated, the empirically identified optimal ratio $r^\ast$ for a given ID dataset is consistently optimal across its corresponding OOD pairings. This suggests that $r^\ast$ captures dataset-specific structure that generalizes across distributional shifts. However, for other OOD datasets (e.g., notMNIST, Omniglot) \cite{bulatov2011notMNIST, lake2015omniglot}, the performance gap between $r^\ast$ and alternative configurations is minimal. To understand this, we examine the generalization bounds described by \cite{tishby2018dlandibp}.

\subsection{Bound Tightness and Task Simplicity}
\label{sec:analysis_generalization}

We analyze the generalization error bound of the form:

\begin{equation}
    \left| I(Z; Y) - \hat{I}(Z; Y) \right| \leq O\left(\frac{l_i |\mathcal{Y}|}{\sqrt{n}} \right),
    \label{eq:ibt_gen_bounds_sub}
\end{equation}
where $I(Z; Y)$ is the true mutual information between representation and label, $\hat{I}(Z; Y)$ its empirical estimate, and the error term depends on latent dimension $l_i$, label set size $|\mathcal{Y}|$, and sample size $n$.

We distinguish between \texttt{Near-OOD} (hard) and \texttt{Far-OOD} (easy) scenarios, where the latter refers to OOD distributions that are clearly separable from the ID distribution (e.g., Omniglot). In these tasks, only minimal information is needed to discriminate between in and out distributions.

\paragraph{Simplicity Implies Low Mutual Information}

For \texttt{Far-OOD} datasets\footnotemark[3], the Bayes-optimal decision boundary may depend on only a small number of features. Consequently, the required mutual information $I(Z; Y)$ is low, and effective detection can be achieved with compressed representations. As a result, the model may perform well across a range of $r$ values, explaining why performance at $r^\ast$ is not uniquely superior.

\paragraph{Error Term Behavior and Model Complexity}

The approximation error in ~\autoref{eq:ibt_gen_bounds_sub} shrinks with simpler models and larger datasets:

\begin{equation}
  \lim_{l_i \to 0,\ |\mathcal{Y}| \to 2,\ n \to \infty} O\left( \frac{l_i |\mathcal{Y}|}{\sqrt{n}} \right) = 0.
\end{equation}

In practice, OOD detection is binary ($|\mathcal{Y}| = 2$), and simple tasks allow for smaller latent dimensions $l_i$. Thus, for well-separated ID/OOD distributions, the bound becomes tight, and empirical estimates $\hat{I}(Z; Y)$ closely approximate true mutual information. That is:

\[
    \hat{I}(Z; Y) \approx I(Z; Y)
\]

\paragraph{Implications for HVAE Configuration}

This analysis clarifies why $r^\ast$ may not dominate in all settings. For easy detection tasks, several configurations suffice, and the choice of $r$ is less sensitive. In contrast, for \textit{Near-OOD} datasets, the model must retain more task-relevant information, and suboptimal $r$ leads to clear degradation, making $r^\ast$ uniquely optimal.

\paragraph{Summary}

The informativeness of $r^\ast$ depends on the complexity of the OOD detection task. For \texttt{Far-OOD} settings, generalization bounds are tight, and many latent allocations perform well. For \textit{Near-OOD}, accurate information retention is critical, and $r^\ast$ yields distinct gains.

\footnotetext[3]{We use ``Far-OOD'' to refer to dataset pairings such as Omniglot and notMNIST, which differ significantly from the ID distribution.}

\section{Conclusion}
\label{sec:conclusion}
We presented a principled framework for allocating latent dimensionality in hierarchical VAEs, introducing the concept of a latent compression ratio $ r $ and demonstrating how it can be optimized under a fixed capacity budget. By leveraging geometric allocation and information-theoretic analysis, we formulated and empirically validated the existence of an optimal compression ratio $ r^\ast $ that balances abstraction and informativeness for out-of-distribution (OOD) detection.

Our experiments show that $ r^\ast $ consistently outperforms other configurations in the compressed setting across all datasets, supporting the hypothesis that hierarchical representational structure directly influences OOD detection efficacy. Furthermore, mutual information estimates reveal that both overcompression and undercompression degrade performance in predictable ways, aligning with our information bottleneck formulation.

This work highlights the utility of compression-aware latent design and motivates future extensions, including adapting the ratio dynamically per layer, exploring task-adaptive priors, and extending these ideas to non-Gaussian latent distributions or multimodal generative tasks.

\section*{Limitations}
\label{sec:limitations}

Our work focuses on improving the performance of Hierarchical Variational Autoencoders (HVAEs) for out-of-distribution (OOD) detection through principled latent dimension allocation. As such, we do not compare against non-hierarchical generative models or discriminative OOD detectors. This is by design: our goal is to analyze and improve the internal structure of HVAEs, rather than to evaluate them as the best-performing class of OOD detectors. Techniques such as MSP, ODIN \citep{hendrycks2016baseline, liang2017odin}, or energy-based scores \citep{liu2020energyOOD} operate at the output level and are orthogonal to the representational improvements we propose.

Another limitation is that our theoretical framework assumes a fixed total latent dimensionality budget and focuses on allocation under this constraint. In settings where total model capacity is unconstrained or where other architectural bottlenecks dominate, our approach may offer less benefit.

Finally, while we observe consistent performance gains across multiple datasets, the optimal ratio $r^\ast$ is determined empirically per dataset. Understanding how $r^\ast$ scales with dataset complexity, modality, or downstream tasks remains an open direction for future work.

{
    \small
    \bibliographystyle{elsarticle-num}
    \bibliography{main}
}

\clearpage

\begin{center}
    \textbf{\huge Supplementary Material}
\end{center}

\section{Proofs}
\begin{proof}[Proof of Existence of $r^\ast$]
\label{sec:r_star_proof}
    
    Given the total latent budget $b$ and number of layers $ N $, the latent dimension $ l_i $ for layer $i$ is:
    \begin{equation}
        l_i = b \cdot \frac{(1 - r) \cdot r^{i-1}}{1 - r^N}.
        \label{eq:appendix_layer_dim}
    \end{equation}
    
    Define the objective function:
    \begin{equation}
        \mathcal{F}(r) = \sum_{i=1}^N f(l_i) = \sum_{i=1}^N f \left( b \cdot \frac{(1 - r) \cdot r^{i-1}}{1 - r^N} \right).
        \label{eq:appendix_objective}
    \end{equation}
    
    Given the properties of \( f \):
    \begin{equation}
        \frac{\partial f(l_i)}{\partial l_i} \leq 0 \text{ for sufficiently large } l_i \text{ (Saturation Property)}.
    \end{equation}
    
    \begin{equation}
        \lim_{l_i \to 0} f(l_i) = \gamma \text{ (Information Loss Property)}.
    \end{equation}
    
    These properties imply that \( f \) is continuous. Specifically:
    \begin{equation}
        \label{appendix_epsilon_delta_continuity}
        \forall \epsilon > 0, \exists \delta > 0 \text{ such that } |l_i - l_i'| < \delta \Rightarrow |f(l_i) - f(l_i')| < \epsilon.
    \end{equation}

    \textbf{Continuity:}
    Since \( f \) is continuous, the composition \( f \left( b \cdot \frac{(1 - r) \cdot r^{i-1}}{1 - r^N} \right) \) is continuous in \( r \). Therefore, \( \mathcal{F}(r) \) is continuous in \( r \) on \( (0, 1] \).
    
    \textbf{Compact Domain:}
    \( r \) is in the interval \( (0, 1] \), which is a compact domain for practical purposes.
    
    \textbf{Boundedness:}
    Given \( f(l_i) \) is bounded above by maximum detection efficacy: \( M \):
    \begin{equation}
        f(l_i) \leq M.
    \end{equation}
    The N latents in the hierarchy are also bounded above:
    \begin{equation}
        \mathcal{F}(r) \leq N \cdot M.
    \end{equation}
    
    \textbf{Existence of Maximum:}
    By the Extreme Value Theorem, a continuous function on a compact set attains its maximum. Since \( \mathcal{F}(r) \) is continuous on \( (0, 1] \), there exists \( r^\ast \in (0, 1] \) such that:
    \begin{equation}
        \mathcal{F}(r^\ast) = \max_{r \in (0, 1]} \mathcal{F}(r).
    \end{equation}
    
    \textbf{Conclusion:}
    The optimal latent dimension ratio \( r^\ast \) that maximizes the objective function is:
    \begin{equation}
        r^\ast = \underset{r \in (0, 1]}{\argmax} \mathcal{F}(r). 
    \end{equation}
    \begin{equation}
        r^\ast = \underset{r \in (0, 1]}{\argmax} \sum_{i=1}^N f \left( b \cdot \frac{(1 - r) \cdot r^{i-1}}{1 - r^N} \right). 
    \end{equation}

    \(\boxed{\text{Q.E.D.}}\)
\end{proof}

% \clearpage 

\begin{proof}[Mutual Information Estimator]
\label{sec:appendix_estimator_derivation}
For a given latent in the hierarchy, i, the Mutual Information $I( X; Z_{i})$ can be expressed as:

\begin{align}
    \label{eq:appendix_mutual_information}
    I(X; Z_i) = D_{\text{KL}} \left( p(x, z_i) \, \middle\| \, p(x) \cdot p(z_i) \right)
\end{align}

We can rewrite the KL Divergence as an expectation over the joint distribution, $p(x, z_{i})$:

$$
     I(X; Z_i) = \mathbb{E}_{p(x, z_i)} \biggl[ \log \frac{p(x, z_{i})}{p(x) \cdot p(z_{i})} \biggr]
$$

We can use the properties of logarithms to rewrite the expectation below:

\begin{align}
     I(X; Z_i) &=  \mathbb{E}_{p(x, z_i)} \left[ \log p(x, z_{i}) \right] \label{eq:appendix_mi_joint} \\
     &- \mathbb{E}_{p(x, z_i)} \left[ \log p(x) \right] \label{eq:appendix_mi_marginal} \\
     &- \mathbb{E}_{p(x, z_i)} \left[ \log p(z_{i}) \right] \label{eq:appendix_mi_prior}
\end{align}

Each of the terms in our expression for $I(X; Z_i)$ is computed in practice as follows:

\begin{description}
    \item  (\autoref{eq:appendix_mi_joint}): The expectation of the log-likelihood under the joint distribution. We sample $z_i$ from the posterior and all $z_{j \neq i}$ from its respective prior. We then compute the log likelhood of the input, X under the output distribution defined by this generative process.
    \item (\autoref{eq:appendix_mi_marginal}): The expectation of the log-likelihood of \(X\) under the marginal distribution. This is intractable, and in practice we use the ELBO ~\autoref{eq:vae_elbo} as surrogate. The KL-Divergence at layer i will 0 since samples are coming from the prior.
    \item (\autoref{eq:appendix_mi_prior}): The expectation of the log-prior over \(z_i\). We estimate the log-likelihood of the samples inferred under the respective prior distribution.
\end{description}

\end{proof}

% \clearpage

\begin{proof}[Mutual Information Relationship with Dimensionality]
\label{sec:appendix_mi_z}

To analyze the effect of latent dimensionality \( |Z| \) on mutual information, consider the definition:
\[
I(X; Z) = H(Z) - H(Z \mid X),
\]
where \( H(Z) \) is the entropy of the latent representation \( Z \), and \( H(Z \mid X) \) is the conditional entropy given the input \( X \).

\begin{itemize}
    \item Effect of \( |Z| \to 0 \) on \( H(Z) \): \\
   If \( Z \in \mathbb{R}^{|Z|} \), the entropy \( H(Z) \) scales with the dimensionality \( |Z| \). For a Gaussian representation \( Z \) with covariance matrix \( \Sigma_Z \), the entropy is:
   \[
   H(Z) = \frac{1}{2} \log \bigl((2\pi e)^{|Z|} \det \Sigma_Z \bigr).
   \]
   As \( |Z| \to 0 \), the space of \( Z \) collapses, resulting in \( H(Z) \to 0 \).

   \item Effect of \( |Z| \to 0 \) on \( H(Z \mid X) \): \\
   The conditional entropy \( H(Z \mid X) \) measures the variability in \( Z \) given \( X \). Since \( H(Z \mid X) \leq H(Z) \), and \( H(Z) \to 0 \) as \( |Z| \to 0 \), it follows that:
   \[
   H(Z \mid X) \to 0.
   \]

    \item Implication for Mutual Information: 
   Substituting into the mutual information definition:
   \[
   I(X; Z) = H(Z) - H(Z \mid X),
   \]
   both terms \( H(Z) \) and \( H(Z \mid X) \) approach zero as \( |Z| \to 0 \). Therefore:
   \[
   I(X; Z) \to 0.
   \]

   \item Effect on \( I(Z; Y) \): 
   Similarly, the ability of \( Z \) to encode task-relevant information about \( Y \) is constrained by its dimensionality. A small \( |Z| \) limits the representation power of \( Z \), resulting in:
   \[
   I(Z; Y) \to 0.
   \]

   \item Generalization Bound Perspective:
   From the generalization bound:
   \[
   O\left(\frac{|Z| |\mathcal{Y}|}{\sqrt{n}}\right),
   \]
   a smaller \( |Z| \) reduces the error term, but as shown above, the reduction in mutual information leads to underfitting, as neither \( I(X; Z) \) (compression) nor \( I(Z; Y) \) (task relevance) is sufficient.

\end{itemize}

\(\boxed{\text{Q.E.D.}}\) When \( |Z| \to 0 \), both \( I(X; Z) \to 0 \) and \( I(Z; Y) \to 0 \), leading to a trivial representation that underfits the data.

\end{proof}

% \clearpage
\section{Decoding Distributions}

In this section, we describe the different decoding distributions considered for the HVAES considered in our experiments (as well as prior work). 

\subsection{Bernoulli Distribution}

The Bernoulli distribution is used for modeling binary data. The decoder outputs logits that are transformed into probabilities for binary classification. The likelihood of observing a binary data point \( x \) given the latent variable \( z \) is:

\begin{equation}
p(x | z) = \text{Bernoulli}(x | \sigma(f(z)))
\end{equation}

where \( \sigma \) is the sigmoid function, and \( f(z) \) represents the neural network output (logits). The likelihood is given by:

\begin{equation}
p(x | z) = \sigma(f(z))^x \cdot (1 - \sigma(f(z)))^{1 - x}
\end{equation}

\subsection{Gaussian Distribution}

The Gaussian distribution is used for modeling continuous data. The decoder outputs parameters for a Gaussian distribution, typically the mean and variance. The likelihood of observing a continuous data point \( x \) given \( z \) is:

\begin{equation}
p(x | z) = \mathcal{N}(x | \mu(z), \sigma^2(z))
\end{equation}

where \( \mu(z) \) and \( \sigma^2(z) \) are the mean and variance parameters output by the neural network.

\subsection{Discretized Logistic Distribution}

The Discretized Logistic distribution is used for modeling continuous data with values constrained to specific intervals. It is particularly useful for approximating pixel values in images. The likelihood for a discrete logistic distribution is:

\begin{equation}
p(x | z) = \frac{\exp\left(\frac{x - \mu(z)}{\sigma(z)}\right)}{\sigma(z) \left[1 + \exp\left(\frac{x - \mu(z)}{\sigma(z)}\right)\right]^2}
\end{equation}

where \( \mu(z) \) and \( \sigma(z) \) are the mean and scale parameters, respectively, output by the neural network.

\subsection{Mixture of Discretized Logistic Distributions}

The Mixture of Discretized Logistic distributions extends the Discretized Logistic distribution by combining multiple logistic distributions with different parameters. This approach allows for capturing more complex data distributions. Given \( K \) components, the mixture likelihood is:

\begin{equation}
p(x | z) = \sum_{k=1}^{K} \pi_k \cdot \frac{\exp\left(\frac{x - \mu_k(z)}{\sigma_k(z)}\right)}{\sigma_k(z) \left[1 + \exp\left(\frac{x - \mu_k(z)}{\sigma_k(z)}\right)\right]^2}
\end{equation}

where \( \pi_k \) are the mixture weights, and \( \mu_k(z) \) and \( \sigma_k(z) \) are the mean and scale parameters for the \( k \)-th logistic component, respectively. The weights \( \pi_k \) are typically constrained to sum to 1.

\section{Reduction Ratio Geometric Series}\label{sec:rrgs}

\subsection{Derivation for \texorpdfstring{$l_{i}$}{li}}

Given the first layer's dimensionality, \( l_{1} \), we can express the dimensionality of the \( i \)-th layer in a hierarchical structure where the dimensions of successive layers follow a geometric series.

\subsubsection{Assumptions}

\begin{description}
    \item The dimensionality of the first layer is \( l_{1} \).
    \item Each subsequent layer has a dimensionality scaled by a factor of \( r \) compared to the previous layer.
    \item The sum of the dimensions of all layers should equal the budget \( b \):
        \begin{equation}
            b = l_{1} + (l_{1} \cdot r) + (l_{1} \cdot r^{2}) + \dots + (l_{1} \cdot r^{N-1})
        \end{equation}
\end{description}

For a fixed budget \( b \) and number of layers \( N \), each layer after \( l_1 \) is scaled by \( r \).

\subsubsection{Geometric Series Formula}
The sum over all the layers can be expressed as the sum of the first \( N \) terms of a geometric series:
        \begin{equation}
           S = l_{1} \cdot \frac{1-r^{N}}{1-r}, \quad (\text{for} \, r \neq 1)
        \end{equation}

We can apply the budget constraint by setting the sum equal to the budget \( b \), i.e. \( b = S \):
        \begin{equation}
           b = l_{1} \cdot \frac{1-r^{N}}{1-r}
           \label{eq:budget_geo_sum}
        \end{equation}

\subsubsection{Solving for \texorpdfstring{$ l_{1} $}{l1}}
Rearranging equation~\eqref{eq:budget_geo_sum} in terms of \( l_{1} \), we get:
        \begin{equation}
           l_{1} = b \cdot \frac{1-r}{1-r^{N}}
        \end{equation}

\subsubsection{Dimensionality of Layer \texorpdfstring{$i$}{i}}
We can now express the dimensionality of any layer \( l_{i} \) in terms of the initial layer \( l_1 \) and the common ratio \( r \).
From the geometric progression, the dimensionality of layer \( i \) is related to the first layer's dimensionality by a factor of \( r^{i-1} \), so:
\begin{align}
    l_{i} = l_{1} \cdot r^{i-1}.
\end{align}

Substituting the expression for \( l_{1} \) from above:
\begin{align}
    l_{i} = \left( b \cdot \frac{1-r}{1-r^{N}} \right) \cdot r^{i-1}.
\end{align}

\subsubsection{Conclusion}

We have derived an expression for the dimensionality of the \(i\)-th layer, \(l_i\), in terms of the budget \(b\), the common ratio \(r\), and the number of layers \(N\). The final expression is:

\begin{equation}
    l_i = \left( b \cdot \frac{(1-r) \cdot  (r^{i-1})}{1-r^{N}} \right) 
\end{equation}

This completes the derivation.

\subsection{Optimization of \texorpdfstring{$r$}{r}}

To find the optimal reduction ratio \( r \) that maximizes the sum of a function \( f \) applied to the dimensionalities of all layers, we start with the expression:

\begin{equation}
    r^{\ast} = \underset{r}{\argmax} \sum_{i=1}^{N} f(l_{i})
    \label{eq:r_star_appx}
\end{equation}

Given the expression for \( l_{i} \) from \autoref{eq:ratio_layer_dim}, we substitute \( l_{i} \) into the optimization problem:

\begin{equation}
    r^{\ast} = \underset{r}{\argmax} \sum_{i=1}^{N} f \biggl( b \cdot \frac{(1 - r) \cdot (r^{i-1})}{1 - r^{N}} \biggr)
    \label{eq:r_star_full_appx}
\end{equation}

This results in the optimization of the reduction ratio \( r \) based on the specific function \( f \) applied to the dimensionalities of the layers.

\par In practice we must ensure that $\forall i, l_{i} \in \mathcal{Z}^{+}$ by rounding the computed dimensions to the nearest integer. Initially, dimensions are calculated based on the top layer size, $l_{1}$ and the compression ratio, r, with a minimum value of 1 for each layer. If the total sum of dimensions does not exactly match the budget, b, we adjusts dimensions by incrementing or decrementing values from the higher layers, ensuring all dimensions remain integers and the total budget is met.

\section{Evaluation}
\subsection{Metrics}
 These metrics are widely adopted in the literature for OOD Detection:
\begin{itemize}
    \item \textbf{FPR80 (False Positive Rate at 80\% True Positive Rate)}: Measures the rate of false positives when the true positive rate is 80\%.
    \item \textbf{AUROC (Area Under the Receiver Operating Characteristic curve)}: Evaluates the ability of the model to distinguish between in-distribution and out-of-distribution samples.
    \item \textbf{AUPRC (Area Under the Precision-Recall Curve)}: Assesses the trade-off between precision and recall, especially useful for imbalanced datasets.
    
\end{itemize}

\subsection{Evaluation Samples}
To compute AUPRC, AUROC and FPR80, we take 10000 samples per dataset, in contrast to \citep{havtorn2021hvaes}
and \citep{li2022adaptiveratio} who take 1000 samples per dataset.

\subsection{IWAE Samples}
We compute our results through 1000 Importance samples \cite{burda2016iwae}.

\section{Model Details}

\begin{table}[ht]
    \centering
    \begin{tabular}{l l}
        \toprule
        Hyperparameter & Setting/Range \\
        \midrule
        GPUs & GTX2080 Ti  \\
        Optimization & Adam \cite{kingma2017adam} \\
        Learning Rate & $3 e^{-4}$ \\
        Batch Size & 128 \\
        Epochs & 2000 \\
        Free bits  & 2) \\
        \bottomrule
    \end{tabular}
    \caption{Hyperparameters for training of HVAE models}
    \label{tab:hvae_hyperparameters}
\end{table}

The hyperparameter settings used to train the hierarchical variational autoencoders used for detection are listed in ~\autoref{tab:hvae_hyperparameters}. Training time for a single HVAE model on one Nvidia GTX 2080 Ti GPU was approximately 48 hours for grayscale images and $>200$ hours for natural images. The  parameterization of the latent variables is also identical to the original setup.

% \clearpage 

\section{Additional Results}

\subsection{Mutual Information Relationship with r}
\par To illustrate the tradeoff between attenuation and information loss, we first explore the relationship between choice of $r$ in the compressed configurations and the corresponding estimates of $I(X; Z_{3})$.

\subsubsection{\texorpdfstring{$I (X; Z_{3})$}{I(X; Z3} Loss}
At r=$(0.1)$, the model exhibits aggressive compression and thus $I(X; Z_{3})$ is at its lowest. Since $r^{\ast}$ is never found at this value, this indicates uniform information loss across all datasets, where the latent representation fails to capture essential information for the detection task.

\subsubsection{\texorpdfstring{$I (X; Z_{3})$}{I (X; Z3)} Attenuation}
As seen in ~\autoref{fig:mi_plot}, $r^{\ast}$ is never the configuration with $\max I(X; Z_{3})$. Additionally we empirically observe that as $r \to 1, I(X; Z_{3} \to \infty)$, however, $r^{\ast}$ again is not found in the configuration closest to $r=1, (r=0.75)$. This corresponds to attenuation where additional representational capacity at $Z_{3}$ does not lead to an improvement in detector performance.

\par It is also noteworthy that as $r \to 1, \Var(I(X; Z_{3})) \to \infty$. The mutual information estimates begin to exhibit higher variance as the compression strategy becomes more relaxed and the corresponding latent becomes larger. This not only highlights a decrease in detection performance, but also an increased potential for capturing "irrelevant" or "useless" information. i.e.. $ r \to 1 \implies I(X; Z_{3}) \to H(Z_{3}) $

% \input{figures/mi_plot}

% input Grayscale Table
\begin{table*}[ht]
    \centering
    
    \caption{Mean and associated standard deviation in accuracy for $\LLRK$ score function across four OOD detection metrics for FashionMNIST and MNIST in vs OOD datasets. Leftmost column indicate values allocated in hierarchy for b=32 and N=3. The optimal scores are highlighted in \textbf{bold.} For Omniglot, since all configurations are found to be optimal none is highlighted.}
    \resizebox{0.95\textwidth}{!}{%
    \label{tab:ratio_gray_appendix}

    \begin{tabular}{|c|c|c|c|c|c|c|c|c|c|c|c|c|c|c|c|c|c|}
                \hline
         & \multicolumn{16}{c|}{Trained on FashionMNIST.}  \\
         \hline
         & \multicolumn{16}{c|}{Control Configurations} \\
        \cline{2-17}
        \scriptsize  OOD & \multicolumn{4}{c|}{MNIST} & \multicolumn{4}{c|}{KMNIST} & \multicolumn{4}{c|}{notMNIST} & \multicolumn{4}{c|}{Omniglot} \\
        \cline{2-17}
        \scriptsize Metric & \scriptsize AUPRC $\uparrow$ & \scriptsize AUROC $\uparrow$  & \scriptsize FPR80  $\downarrow$ & \scriptsize FPR95 $\downarrow$ &  \scriptsize AUPRC $\uparrow$ & \scriptsize AUROC $\uparrow$  & \scriptsize FPR80  $\downarrow$ & \scriptsize FPR95 $\downarrow$  &  \scriptsize AUPRC $\uparrow$ & \scriptsize AUROC $\uparrow$  & \scriptsize FPR80  $\downarrow$ & \scriptsize FPR95 $\downarrow$ &  \scriptsize AUPRC $\uparrow$ & \scriptsize AUROC $\uparrow$  & \scriptsize FPR80  $\downarrow$ & \scriptsize FPR95 $\downarrow$ \\
        \hline
         $6 \rightarrow 13 \rightarrow 13$ &\errpm{0.9143}{0.0159} & \errpm{0.9052}{0.0228} & \errpm{0.1442}{0.0316} & \errpm{0.4757}{0.1378} & \errpm{0.9714}{0.0083} & \errpm{0.9678}{0.0092} & \errpm{0.028}{0.0104} & \errpm{0.1937}{0.0661} & \errpm{0.9998}{0.0001} & \errpm{0.9998}{0.0001} & \errpm{0.}{0.} & \errpm{0.}{0.} & \errpm{1.}{0.} & \errpm{1.}{0.} & \errpm{0.}{0.} & \errpm{0.}{0.} \\ 
         $8 \rightarrow 10 \rightarrow 14$ & \textbf{\errpm{0.926}{0.0203}} & \textbf{\errpm{0.9254}{0.0223} }& \errpm{0.122}{0.0374} & \textbf{\errpm{0.3116}{0.1029}} & \errpm{0.9805}{0.0112} & \errpm{0.9779}{0.0135} & \errpm{0.0143}{0.0108} & \errpm{0.1429}{0.0958} & \errpm{0.9998}{0.0001} & \errpm{0.9998}{0.0001} & \errpm{0.}{0.} & \errpm{0.0003}{0.0006} & \errpm{1.}{0.} & \errpm{1.}{0.} & \errpm{0.}{0.} & \errpm{0.}{0.} \\ 
        $8 \rightarrow 16 \rightarrow 8$ & \errpm{0.9188}{0.0216} & \errpm{0.9169}{0.0198} & \errpm{0.1494}{0.0303} & \errpm{0.3672}{0.0605} & \textbf{\errpm{0.9883}{0.001}} & \textbf{\errpm{0.987}{0.0012}} & \textbf{\errpm{0.0033}{0.0023}} & \textbf{\errpm{0.0817}{0.0011}} & \textbf{\errpm{1.}{0.}} & \textbf{\errpm{1.}{0.}} & \textbf{\errpm{0.}{0.}} & \textbf{\errpm{0.}{0.}} & \errpm{1.}{0.} & \errpm{1.}{0.} & \errpm{0.}{0.} & \errpm{0.}{0.} \\ 
         $10 \rightarrow 11 \rightarrow 11$ &\errpm{0.9113}{0.0269} & \errpm{0.9137}{0.0036} &\textbf{ \errpm{0.1037}{0.0173}} & \errpm{0.4143}{0.1698} & \errpm{0.9722}{0.0117} & \errpm{0.9688}{0.0141} & \errpm{0.027}{0.0142} & \errpm{0.1976}{0.0911} & \errpm{0.9998}{0.0002} & \errpm{0.9998}{0.0002} & \errpm{0.}{0.} & \errpm{0.0007}{0.0011} & \errpm{1.}{0.} & \errpm{1.}{0.} & \errpm{0.}{0.} & \errpm{0.}{0.} \\ 
         $13 \rightarrow 6 \rightarrow 13$ & \errpm{0.6996}{0.2299} & \errpm{0.5686}{0.3296} & \errpm{0.7298}{0.2573} & \errpm{0.9666}{0.0295} & \errpm{0.7737}{0.1927} & \errpm{0.7407}{0.1951} & \errpm{0.4251}{0.2584} & \errpm{0.7777}{0.166} & \errpm{0.9797}{0.0244} & \errpm{0.9783}{0.0244} & \errpm{0.0257}{0.0445} & \errpm{0.1048}{0.1315} & \errpm{1.}{0.} & \errpm{1.}{0.} & \errpm{0.}{0.} & \errpm{0.}{0.} \\ 
        \hline
        \hline
        & \multicolumn{16}{c|}{Compressed Configurations} \\
        \cline{2-17}
          $28 \rightarrow 3 \rightarrow 1$ & \errpm{0.774}{0.0864} & \errpm{0.715}{0.0741} & \errpm{0.58}{0.1338} & \errpm{0.981}{0.0137} & \errpm{0.85}{0.0434} & \errpm{0.821}{0.0302} & \errpm{0.347}{0.0088} & \errpm{0.79}{0.0494} & \errpm{0.984}{0.0076} & \errpm{0.98}{0.0102} & \errpm{0.004}{0.0045} & \errpm{0.13}{0.0935} & \errpm{1.}{0.} & \errpm{1.}{0.} & \errpm{0.}{0.} & \errpm{0.}{0.} \\
           $24 \rightarrow 6 \rightarrow 2$ & \errpm{0.876}{0.0526} & \errpm{0.858}{0.022} & \errpm{0.21}{0.0355} & \errpm{0.711}{0.1228} & \errpm{0.93}{0.0196} & \errpm{0.916}{0.0159} & \errpm{0.113}{0.044} & \errpm{0.496}{0.067} & \errpm{0.997}{0.0012} & \errpm{0.996}{0.0014} & \errpm{0.}{0.0006} & \errpm{0.011}{0.0061} & \errpm{1.}{0.} & \errpm{1.}{0.} & \errpm{0.}{0.} & \errpm{0.}{0.} \\
           $18 \rightarrow 9 \rightarrow 5$ & \errpm{0.886}{0.0048} & \errpm{0.88}{0.0074} & \errpm{0.182}{0.0198} & \errpm{0.533}{0.0691} & \errpm{0.968}{0.0042} & \errpm{0.963}{0.0035} & \errpm{0.037}{0.0114} & \errpm{0.212}{0.023} & \errpm{1.}{0.0002} & \errpm{0.999}{0.0003} & \errpm{0.}{0.} & \errpm{0.}{0.} & \errpm{1.}{0.} & \errpm{1.}{0.} & \errpm{0.}{0.} & \errpm{0.}{0.} \\
          $14 \rightarrow 10 \rightarrow 8$ &  \textbf{\errpm{0.938}{0.0144}} & \textbf{\errpm{0.948}{0.008} }& \textbf{\errpm{0.085}{0.0077}} & \textbf{\errpm{0.168}{0.0534} }& \textbf{\errpm{0.989}{0.0019}} & \textbf{\errpm{0.988}{0.0021}} & \textbf{\errpm{0.007}{0.0006}} & \textbf{\errpm{0.064}{0.0181}} & \textbf{\errpm{1.}{0.}} & \textbf{\errpm{1.}{0.}} & \textbf{\errpm{0.}{0.}} & \textbf{\errpm{0.}{0.}} & \errpm{1.}{0.} & \errpm{1.}{0.} & \errpm{0.}{0.} & \errpm{0.}{0.} \\
          \hline
        \hline
        \multirow{4}{*}{} & \multicolumn{16}{c|}{Trained on MNIST.}  \\
         \hline
         & \multicolumn{16}{c|}{Control Configurations} \\
         %\hline
         \cline{2-17}
          & \multicolumn{4}{c|}{FashionMNIST} & \multicolumn{4}{c|}{KMNIST} & \multicolumn{4}{c|}{notMNIST} & \multicolumn{4}{c|}{Omniglot} \\
          \hline
           $6 \rightarrow 13 \rightarrow 13$ & \errpm{0.8751}{0.013} & \errpm{0.8872}{0.0065} & \errpm{0.2055}{0.0081} & \errpm{0.4028}{0.0021} & \errpm{0.8778}{0.005} & \errpm{0.8815}{0.0049} & \errpm{0.2227}{0.0061} & \errpm{0.4489}{0.0093} & \errpm{0.9927}{0.0001} & \errpm{0.9934}{0.0002} & \errpm{0.0107}{0.} & \errpm{0.0329}{0.0025} & \errpm{1.}{0.} & \errpm{1.}{0.} & \errpm{0.}{0.} & \errpm{0.}{0.}\\
         $8 \rightarrow 10 \rightarrow 14$& \errpm{0.7875}{0.0374} & \errpm{0.837}{0.0218} & \errpm{0.2732}{0.0285} & \errpm{0.4481}{0.032} & \errpm{0.8059}{0.0209} & \errpm{0.8305}{0.0201} & \errpm{0.3112}{0.0497} & \errpm{0.5505}{0.0519} & \errpm{0.9767}{0.0111} & \errpm{0.9827}{0.007} & \errpm{0.029}{0.0098} & \errpm{0.0615}{0.0183} & \errpm{1.}{0.} & \errpm{1.}{0.} & \errpm{0.}{0.} & \errpm{0.}{0.}\\
         $8 \rightarrow 16 \rightarrow 8$ &\errpm{0.867}{0.1228} & \errpm{0.8918}{0.0882} & \errpm{0.18}{0.1429} & \errpm{0.3236}{0.1732} & \errpm{0.861}{0.1172} & \errpm{0.8727}{0.0998} & \errpm{0.2282}{0.1841} & \errpm{0.4294}{0.2409} & \errpm{0.982}{0.0242} & \errpm{0.9869}{0.0165} & \errpm{0.0186}{0.0249} & \errpm{0.0482}{0.0515} & \errpm{1.}{0.} & \errpm{1.}{0.} & \errpm{0.}{0.} & \errpm{0.}{0.}\\
         $10 \rightarrow 11 \rightarrow 11$ & \textbf{\errpm{0.9157}{0.0439}} & \textbf{\errpm{0.9206}{0.0395}} & \textbf{\errpm{0.1469}{0.0744}} & \textbf{\errpm{0.3231}{0.1202} }& \textbf{\errpm{0.9198}{0.0459}} & \textbf{\errpm{0.9198}{0.0449}} & \textbf{\errpm{0.1449}{0.1}} & \textbf{\errpm{0.332}{0.1357}} & \textbf{\errpm{0.9959}{0.0029}} & \textbf{\errpm{0.9962}{0.0026}} & \textbf{\errpm{0.0055}{0.0048}} & \textbf{\errpm{0.0215}{0.0127}} & \errpm{1.}{0.} & \errpm{1.}{0.} & \errpm{0.}{0.} & \errpm{0.}{0.}\\
        $13 \rightarrow 6 \rightarrow 13$ & \errpm{0.8986}{0.0118} & \errpm{0.9007}{0.007} & \errpm{0.1886}{0.0085} & \errpm{0.3965}{0.0098} & \errpm{0.8878}{0.0109} & \errpm{0.8884}{0.008} & \errpm{0.2077}{0.0126} & \errpm{0.4111}{0.0131} & \errpm{0.9944}{0.001} & \errpm{0.9949}{0.0008} & \errpm{0.0065}{0.002} & \errpm{0.0277}{0.0011} & \errpm{1.}{0.} & \errpm{1.}{0.} & \errpm{0.}{0.} & \errpm{0.}{0.}\\
        \hline
        \hline
        & \multicolumn{16}{c|}{Compressed Configurations} \\
        \cline{2-17}
         $28 \rightarrow 3 \rightarrow 1$ &\errpm{0.94}{0.0134} & \errpm{0.942}{0.0111} & \errpm{0.097}{0.0259} & \errpm{0.249}{0.0318} & \errpm{0.948}{0.0124} & \errpm{0.948}{0.0106} & \errpm{0.075}{0.019} & \errpm{0.26}{0.0436} & \errpm{0.997}{0.0012} & \errpm{0.997}{0.0011} & \errpm{0.003}{0.0028} & \errpm{0.018}{0.0068} & \errpm{1.}{0.} & \errpm{1.}{0.} & \errpm{0.}{0.} & \errpm{0.}{0.}\\
         $24 \rightarrow 6 \rightarrow 2$ & \errpm{0.973}{0.0136} & \errpm{0.973}{0.0126} & \errpm{0.039}{0.0269} & \errpm{0.145}{0.0475} & \errpm{0.975}{0.0153} & \errpm{0.974}{0.0148} & \errpm{0.03}{0.0279} & \errpm{0.149}{0.0618} & \errpm{0.999}{0.0007} & \errpm{0.999}{0.0007} & \errpm{0.}{0.} & \errpm{0.006}{0.0051} & \errpm{1.}{0.} & \errpm{1.}{0.} & \errpm{0.}{0.} & \errpm{0.}{0.}\\
          $18 \rightarrow 9 \rightarrow 5$ & \textbf{\errpm{0.999}{0.0003}} & \textbf{\errpm{0.999}{0.0003}} & \textbf{\errpm{0.}{0.}} & \textbf{\errpm{0.001}{0.0006}} & \textbf{\errpm{1.}{0.0001}} &\textbf{ \errpm{1.}{0.0001}} &\textbf{ \errpm{0.}{0.}} &\textbf{ \errpm{0.}{0.}} & \textbf{\errpm{1.}{0.}} & \textbf{\errpm{1.}{0.}} & \textbf{\errpm{0.}{0.}} & \textbf{\errpm{0.}{0.} } & \errpm{1.}{0.} & \errpm{1.}{0.} & \errpm{0.}{0.} & \errpm{0.}{0.}\\
          $14 \rightarrow 10 \rightarrow 8$ & \errpm{0.988}{0.0096} & \errpm{0.988}{0.0097} & \errpm{0.015}{0.013} & \errpm{0.067}{0.0576} & \errpm{0.993}{0.0061} & \errpm{0.993}{0.0064} & \errpm{0.004}{0.0039} & \errpm{0.042}{0.0382} & \errpm{1.}{0.0001} & \errpm{1.}{0.0001} & \errpm{0.}{0.} & \errpm{0.}{0.} & \errpm{1.}{0.} & \errpm{1.}{0.} & \errpm{0.}{0.} & \errpm{0.}{0.}\\
         \hline
    \end{tabular}
    }
\end{table*}

% input Color Table
\begin{table*}[ht]
    \centering

    \caption{Mean $\pm$ std $\LLRK$ score function across four OOD detection metrics for CIFAR10 and SVHN in vs OOD datasets. Leftmost column indicate values allocated in hierarchy for b=228 and N=3.}
    
    \resizebox{\textwidth}{!}{%
    \label{tab:ratio_natural_appendix}

    \begin{tabular}{|c|c|c|c|c|c|c|c|c|c|}
                \hline
         & \multicolumn{8}{c|}{Trained on CIFAR10}  \\
         \hline
         & \multicolumn{8}{c|}{Control Configurations} \\
        \cline{2-9}
        \scriptsize  OOD & \multicolumn{4}{c|}{SVHN} & \multicolumn{4}{c|}{CelebA} \\
        \cline{2-9}
        \scriptsize Metric & \scriptsize AUPRC $\uparrow$ & \scriptsize AUROC $\uparrow$  & \scriptsize FPR80  $\downarrow$ & \scriptsize FPR95 $\downarrow$ &  \scriptsize AUPRC $\uparrow$ & \scriptsize AUROC $\uparrow$  & \scriptsize FPR80  $\downarrow$ & \scriptsize FPR95 $\downarrow$  \\
        \hline
        $42 \rightarrow 91 \rightarrow 91$ &\errpm{0.369}{0.007} & \errpm{0.287}{0.021} & \errpm{0.908}{0.012} & \errpm{0.976}{0.006} & \errpm{0.765}{0.008} & \errpm{0.733}{0.012} & \errpm{0.524}{0.044} & \errpm{0.892}{0.012}\\ 
         $56 \rightarrow 80 \rightarrow 98$ & \errpm{0.373}{0.013} & \errpm{0.263}{0.013} & \errpm{0.959}{0.014} & \errpm{0.987}{0.01} & \errpm{0.781}{0.049} & \errpm{0.764}{0.048} & \errpm{0.467}{0.081} & \errpm{0.761}{0.08}\\
         $56 \rightarrow 112 \rightarrow 56$ & \errpm{0.595}{0.196} & \errpm{0.574}{0.234} & \errpm{0.637}{0.226} & \errpm{0.821}{0.151} & \errpm{0.839}{0.086} & \errpm{0.828}{0.106} & \errpm{0.324}{0.226} & \errpm{0.611}{0.259}\\
         $70 \rightarrow 77 \rightarrow 77$ &\errpm{0.708}{0.287} & \errpm{0.682}{0.322} & \errpm{0.46}{0.363} & \errpm{0.665}{0.268} & \errpm{0.861}{0.094} & \errpm{0.851}{0.107} & \errpm{0.262}{0.21} & \errpm{0.564}{0.28}\\
         $91 \rightarrow 42 \rightarrow 91$ &  \errpm{0.871}{0.014} & \errpm{0.871}{0.014} & \errpm{0.255}{0.037} & \errpm{0.47}{0.032} & \errpm{0.849}{0.017} & \errpm{0.823}{0.019} & \errpm{0.361}{0.054} & \errpm{0.689}{0.05}\\
         \hline
         & \multicolumn{8}{c|}{Compressed Configurations} \\
         \hline
         $202 \rightarrow 20 \rightarrow 2$ & \errpm{0.401}{0.023} & \errpm{0.31}{0.033} & \errpm{0.927}{0.034} & \errpm{0.98}{0.017} & \errpm{0.715}{0.033} & \errpm{0.704}{0.017} & \errpm{0.575}{0.041} & \errpm{0.856}{0.032}\\
         $170 \rightarrow 43 \rightarrow 11$ & \textbf{\errpm{0.886}{0.055} }&\textbf{ \errpm{0.88}{0.059}} &\textbf{ \errpm{0.236}{0.123}} & \errpm{0.46}{0.178} & \textbf{\errpm{0.965}{0.011}} & \textbf{\errpm{0.965}{0.013}} & \textbf{\errpm{0.05}{0.02} }& \textbf{\errpm{0.178}{0.072}} \\
         $128 \rightarrow 64 \rightarrow 32$ &\errpm{0.883}{0.114} & \errpm{0.873}{0.127} & \errpm{0.238}{0.251} & \textbf{\errpm{0.428}{0.311} }& \errpm{0.931}{0.015} & \errpm{0.931}{0.014} & \errpm{0.115}{0.012} & \errpm{0.321}{0.056}\\
         $97 \rightarrow 73 \rightarrow 54$ & \errpm{0.855}{0.082} & \errpm{0.846}{0.085} & \errpm{0.312}{0.16} & \errpm{0.523}{0.157} & \errpm{0.87}{0.04} & \errpm{0.858}{0.048} & \errpm{0.271}{0.106} & \errpm{0.569}{0.144}\\
        \hline
    \end{tabular}
    }
\end{table*}

\subsection{Mismatch In Decoding Distributions}
The results reported by the original Hierarchial VAE paper on FashionMNIST and CIFAR10 are for the Bernoulli and MixtureofLogisticDistribution decoders respectively. We use a more flexible distribution for the grayscale datasets.

\section{Threats to Validity}

\subsection{Internal}
\begin{itemize}
    \item \textbf{Selection bias:} The choice of specific dataset pairs (FashionMNIST/MNIST and CIFAR-10/SVHN) while similar to prior work, might not represent the full spectrum of possible in-distribution and out-of-distribution relationships.

\end{itemize}

\subsection{External}
\begin{itemize}
    \item  \textbf{Limited generalizability:} The results might not generalize well to other types of datasets or domains beyond image classification tasks.
\end{itemize}

\subsection{Construct}
\begin{itemize}
    \item  Given the absence of a universally accepted metric for evaluating OOD detection efficacy, we address construct validity by employing a comprehensive set of performance metrics. Specifically, we use AUPRC, AUROC, FPR80 and FPR95 to assess the effectiveness of our technique.

\end{itemize}

\end{document}